\documentclass[11pt]{article}

\usepackage[margin=1in]{geometry}
\usepackage[authoryear,round]{natbib}
\usepackage{amsmath,amssymb,amsfonts,amsthm,mathtools}
\usepackage{booktabs}
\usepackage{microtype}
\usepackage{xcolor}
\usepackage{enumitem}
\usepackage{algorithm}
\usepackage{algpseudocode}
\usepackage{hyperref}
\hypersetup{colorlinks=true,linkcolor=blue,citecolor=blue,urlcolor=blue}
\usepackage{cleveref}
\usepackage{graphicx}
\usepackage{dsfont}
\usepackage{bbm}
\usepackage{url}
\usepackage{multirow}

\newtheorem{theorem}{Theorem}[section]
\newtheorem{lemma}[theorem]{Lemma}
\newtheorem{proposition}[theorem]{Proposition}
\newtheorem{corollary}[theorem]{Corollary}
\newtheorem{definition}[theorem]{Definition}

\newtheorem{remark}[theorem]{Remark}
\newtheorem{problem}[theorem]{Problem}

\DeclareMathOperator{\diag}{diag}
\DeclareMathOperator{\rank}{rank}
\DeclareMathOperator{\Span}{span}

\DeclareMathOperator{\Orb}{Orb}

\newcommand{\R}{\mathbb{R}}
\newcommand{\C}{\mathbb{C}}

\newcommand{\F}{\mathbb{F}}

\newcommand{\od}{\odot}
\newcommand{\one}{\mathds{1}}
\newcommand{\norm}[1]{\left\lVert #1\right\rVert}

\newcommand{\transfer}{\mathcal{T}}

\newcommand{\bigO}{\mathcal{O}}

\title{\bf Semidirect Fourier Delta Attention:\\ Phase-Controlled Delta Memory with Constructive Chunk-WY Kernels}
\author{Tiantian Zhang (Crystal)\\ Columbia University\\ \texttt{t.zhang8@columbia.edu}}
\date{\today}

\begin{document}
\maketitle

\begin{abstract}
Linear attention replaces the growing KV cache of softmax attention with a fixed recurrent state, but fixed-state memory makes exact copying, algorithmic state tracking, and long-context retrieval difficult. We introduce \emph{Semidirect Fourier Delta Attention} (SFDA), a phase-controlled delta-rule layer that strictly contains Kimi Delta Attention (KDA). SFDA replaces KDA's purely real diagonal decay with a block-rotational Fourier control operator. In complex coordinates the recurrent update is
\[
S_t=(I-\beta_t k_tk_t^*)\Lambda_tS_{t-1}+\beta_tk_tv_t^*,
\qquad
\Lambda_t=\diag(\alpha_t\odot e^{i\theta_t}).
\]
The central technical result is a constructive chunk-WY theorem. For a left-to-right product $P_t=A_tA_{t-1}\cdots A_1$ with $A_t=\Lambda_t-u_tr_t^*$, the factors can be updated explicitly as
\[
P_t=\Gamma_t-Y_tM_tW_t^*,
\]
where the rank grows only inside a fixed chunk. This turns closure under composition into an implementable chunk-local recursion, gives an exact affine write summary, and yields formal cost and stability bounds. We also characterize compact expressivity for the actual kernel-compatible family: each chunk is a structured phase/control product plus a rank-$C$ correction, which compactly realizes product cyclic counters, dihedral phase-orientation memory, register/reset memories, and bounded stacks under the corresponding structured controls. We further resolve a finite-state expressivity question for a broader generalized tied-write affine template: every deterministic finite automaton is exactly realized by one-hot lifting and setting the write coefficient to zero. We validate the algebra with numerical theorem checks and toy cyclic-memory experiments, while leaving fused Triton/CUDA speed and large-scale hybrid-ratio comparisons to subsequent systems work.
\end{abstract}

\section{Introduction}

A recurring theme in sequence modeling is the tension between \emph{parallel shortcut computation} and \emph{recurrent state tracking}. Shallow Transformers can simulate finite automata by composing transition functions in parallel \citep{liu2023shortcuts}, but shortcut solutions can be brittle under distribution shift and length extrapolation. Linear attention and gated recurrent attention move in the opposite direction: they maintain a fixed-size state that updates recurrently, which improves inference efficiency but constrains long-context retrieval and precise algorithmic memory.

Kimi Delta Attention (KDA) is a recent and strong example of this second direction. It maintains a matrix-valued memory state
\begin{equation}
S_t=(I-\beta_t k_t k_t^*)\diag(\alpha_t)S_{t-1}+\beta_t k_t v_t^*,
\label{eq:intro-kda}
\end{equation}
and obtains practical efficiency from a chunkwise WY-style representation of diagonal decay plus rank-one delta corrections \citep{kimi2025linear}. KDA's empirical success suggests a question that is both theoretical and systems-oriented:
\begin{quote}
Can we add algebraic phase/control dynamics to KDA-style recurrent memory without destroying the structured chunk kernel?
\end{quote}

We answer this question for a restricted but hardware-friendly phase family. The proposed layer, \emph{Semidirect Fourier Delta Attention} (SFDA), inserts a block-rotational control matrix before the delta correction. In real coordinates,
\begin{equation}
S_t=(I-\beta_t k_tk_t^*)U_tD_tS_{t-1}+\beta_t k_tv_t^*,
\label{eq:intro-sfda-real}
\end{equation}
where $U_t$ is a direct sum of $2\times2$ rotations and $D_t$ is a positive block/diagonal decay. Under the identification $\R^{2m}\simeq \C^m$, this becomes
\begin{equation}
S_t=(I-\beta_t k_tk_t^*)\Lambda_tS_{t-1}+\beta_t k_tv_t^*,
\qquad
\Lambda_t=\diag(\alpha_t\od e^{i\theta_t}).
\label{eq:intro-sfda-complex}
\end{equation}
Thus the algebraic control is a learnable Fourier phase, while the delta correction remains rank one.

The nontrivial part is not the phase parameterization alone. If the phase-control matrix were dense, the kernel would collapse into dense $d\times d$ products. The crucial design choice is to restrict $U_tD_t$ to aligned $2\times2$ rotation-decay blocks, equivalently a complex diagonal $\Lambda_t$. Then every token transition has the form
\begin{equation}
A_t=\Lambda_t-u_tw_t^*,
\label{eq:intro-diag-rank-one}
\end{equation}
that is, phase-decay plus one rank-one delta correction.

\paragraph{Contributions.}
This preprint formalizes the semidirect/Fourier automata idea as a theory-and-kernel framework for phase-controlled delta memory.
\begin{enumerate}[leftmargin=*]
    \item We define tied SFDA as the main architecture and a write-decoupled variant for exact automata constructions. KDA is recovered by setting $U_t=I$, equivalently $\theta_t=0$.
    \item We prove that the phase-control part exactly implements bounded cyclic memory: rotations maintain modular counters and phase trackers with constant norm.
    \item We prove a constructive left-to-right chunk-WY theorem. For $P_t=A_tA_{t-1}\cdots A_1$ with $A_t=\Lambda_t-u_tw_t^*$, the factors $\Gamma_t,Y_t,M_t,W_t$ admit an explicit recursion. This is stronger than an existential closure theorem and is the main kernel theorem.
    \item We prove an exact affine chunk transfer, including the additive write summary
    \[
    B_C=\sum_{i=1}^C A_{C:i+1}B_i,
    \]
    and show how to compute it by a zero-initial-state recurrence.
    \item We characterize compact kernel-compatible expressivity inside a chunk: a length-$C$ SFDA product is a structured phase/control product plus a rank-$C$ correction.  This proves compact exact realizations for product cyclic counters, dihedral phase-orientation memory, register/reset memories, and bounded stacks under the corresponding structured controls.
    \item We give a reader-friendly finite-state theorem: every deterministic finite automaton is exactly realized by a generalized tied-write affine memory system after one-hot lifting. The construction sets the write coefficient to zero, so the erase-write coupling is not an expressivity obstruction for finite state machines. The remaining question is efficient compressed realization, not exact finite-state realizability.
    \item We add an orbit-restricted tied-write compilation theorem. It shows when a decoupled write direction can be routed from a tied key using only the architecture's structured control family, and it explicitly separates contribution routing from full-state compilation with scratch/protected channels.
    \item We include the complex-to-real implementation, formal cost bounds, a numerical stability lemma, and a rank-growth discipline: rank grows only inside fixed chunks, not across the whole sequence.
\end{enumerate}

The main empirical target is not merely to beat KDA on toy automata. The long-term systems target is to test whether the stronger SFDA linear layer can support a higher linear-to-global attention ratio, for example $7{:}1$ or $15{:}1$, while matching or exceeding a KDA-style $3{:}1$ hybrid under equal compute.

\section{Background}
\label{sec:background}

\subsection{Automata, shortcuts, and recurrent memory}

A finite semiautomaton $\mathcal{A}=(Q,\Sigma,\delta)$ evolves by
\begin{equation}
q_t=\delta(q_{t-1},\sigma_t),\qquad q_t\in Q,\ \sigma_t\in\Sigma.
\end{equation}
The per-symbol maps $\delta(\cdot,\sigma)$ generate a finite transformation semigroup. Existing shortcut results show that Transformers can simulate any finite semiautomaton by parallel prefix composition in $O(\log T)$ depth, and that some solvable semiautomata admit constant-depth shortcuts through Krohn--Rhodes-style decompositions \citep{liu2023shortcuts}. However, those shallow shortcuts may fail to generalize to unseen lengths or shifted input distributions, motivating architectures with explicit recurrent state.

\paragraph{Finite automata as ordinary matrix recurrences.}
No automata-theory background is needed for the finite-state expressivity theorem used later. If $Q=\{q_1,\ldots,q_N\}$, encode state $q_i$ by the one-hot vector $e_i\in\R^N$. For each input symbol $\sigma$, define a matrix $P_\sigma$ by
\begin{equation}
P_\sigma e_i=e_j\quad\text{whenever}\quad \delta(q_i,\sigma)=q_j.
\label{eq:background-onehot}
\end{equation}
Then the automaton recurrence is exactly the linear recurrence
\begin{equation}
s_t=P_{\sigma_t}s_{t-1}.
\label{eq:background-linear-dfa}
\end{equation}
The matrices $P_\sigma$ are deterministic one-hot transition matrices: each column has exactly one nonzero entry, but the matrix need not be invertible. For example, a reset transition can map many states to the same next state, so $P_\sigma$ is not necessarily a permutation matrix. This elementary one-hot view is the only finite-automata fact required for the tied-write realization result in \cref{sec:finite-state-main,app:automata}.

The algebraic object we use is a semidirect product
\begin{equation}
N\rtimes H,
\qquad
(n_2,h_2)(n_1,h_1)=(n_2+h_2\cdot n_1,\ h_2h_1),
\label{eq:semidirect-law}
\end{equation}
where $H$ is a control group acting linearly on a memory monoid or vector space $N$. This structure separates two operations: memory storage and control of how memory is transported.

\subsection{KDA as affine recurrent memory}

KDA updates a memory matrix $S_t\in\F^{d\times d_v}$ by \cref{eq:intro-kda}. Abstractly, each token defines an affine map
\begin{equation}
S\mapsto A_tS+B_t,
\qquad
A_t=(I-\beta_tk_tk_t^*)D_t,
\qquad
B_t=\beta_tk_tv_t^*,
\label{eq:affine-token}
\end{equation}
where $D_t=\diag(\alpha_t)$. Affine maps compose by
\begin{equation}
(A_2,B_2)\circ(A_1,B_1)=(A_2A_1,\ A_2B_1+B_2).
\label{eq:affine-law}
\end{equation}
Thus scanability is mathematically immediate. The hard part is efficiency: dense products $A_2A_1$ cost $O(d^3)$, so a useful architecture must preserve a diagonal-plus-low-rank structure inside chunks.

KDA is efficient because $A_t$ is diagonal decay plus a rank-one correction. Its chunkwise formulation stores a compact WY-like representation of the product and uses triangular solves or UT-style transforms to avoid dense matrix products \citep{kimi2025linear,bischof1987wy}. Our goal is to add phase/control dynamics while preserving this same structural invariant.

\subsection{Relation to RoPE, SSMs, DPLR, and KDA}
\label{sec:relation-ssm}

It is helpful to separate where rotations occur.  Rotary positional embeddings rotate query/key features to encode relative position.  SFDA instead rotates the recurrent memory transition itself:
\begin{equation}
S_t=(I-\beta_tk_tk_t^*)\Lambda_tS_{t-1}+\beta_tk_tv_t^*.
\end{equation}
Thus SFDA phases are state-transition phases, not merely query/key positional phases.

\begin{table}[h]
\centering
\caption{Theory-level comparison of transition invariants.}
\label{tab:ssm-comparison}
\begin{tabular}{p{0.20\linewidth}p{0.67\linewidth}}
\toprule
Method family & State transition invariant \\
\midrule
Complex diagonal SSM/S4D & $x_t=A x_{t-1}+Bu_t$ with static or slowly varying diagonal/structured $A$ \\
Mamba-style selective SSM & $x_t=A_t x_{t-1}+B_tu_t$ with input-dependent selective parameters \\
KDA & $S_t=(I-\beta_tk_tk_t^*)D_tS_{t-1}+\beta_tk_tv_t^*$ \\
SFDA & $S_t=(I-\beta_tk_tk_t^*)\Lambda_tS_{t-1}+\beta_tk_tv_t^*$ with phase-decay $\Lambda_t$ \\
General DPLR & $S_t=(D_t-a_tb_t^*)S_{t-1}+B_t$, usually more expressive but harder to fuse efficiently \\
\bottomrule
\end{tabular}
\end{table}

SFDA differs from complex diagonal SSMs because it retains the delta-rule rank-one correction.  It differs from KDA because the real diagonal decay is upgraded to complex phase-decay.  It differs from unrestricted DPLR because the low-rank term is tied to the delta rule and therefore admits the constructive chunk-WY invariant.

\begin{proposition}[Interpolation between KDA and complex diagonal SSMs]
\label{prop:interpolation}
If $\theta_t=0$, SFDA reduces to KDA.  If $\beta_t=0$, SFDA reduces to a complex diagonal state-space recurrence $S_t=\Lambda_tS_{t-1}$.  Hence SFDA interpolates between KDA-style delta memory and oscillatory diagonal SSM-style phase dynamics.
\end{proposition}

\subsection{Why arbitrary orthogonal control is not acceptable}

A tempting extension is to replace $D_t$ by a learned dense orthogonal matrix. That destroys the kernel: products of dense orthogonal matrices with rank-one corrections do not retain a cheap diagonal-plus-low-rank representation. We therefore restrict the control operator to block rotations
\begin{equation}
U_t=\bigoplus_{j=1}^{d/2}R(\theta_{t,j}),
\qquad
R(\theta)=
\begin{bmatrix}
\cos\theta & -\sin\theta\\
\sin\theta & \cos\theta
\end{bmatrix}.
\label{eq:block-rot}
\end{equation}
With paired channels interpreted as complex coordinates, $U_tD_t$ is diagonal. This is the reason SFDA can add phase memory while retaining a KDA-like kernel path. The design is also aligned with the common interpretation of rotary positional encodings as products of $2\times2$ rotations \citep{su2024roformer}.

\section{Semidirect Fourier Delta Attention}
\label{sec:method}

We now define the layer. Let $x_t\in\R^{d_{\mathrm{model}}}$ be the token representation. For each head, projections produce $q_t,k_t\in\F^d$, $v_t\in\F^{d_v}$, a scalar step size $\beta_t\in[0,1]$, channel-wise decay $\alpha_t\in[0,1]^d$, and phase $\theta_t\in[-\theta_{\max},\theta_{\max}]^d$ in complex coordinates. The main architecture in this paper is the tied SFDA recurrence
\begin{align}
S_t&=(I-\beta_tk_tk_t^*)\Lambda_tS_{t-1}+\beta_tk_tv_t^*,
\label{eq:sfda}\\
\Lambda_t&=\diag(\alpha_t\od e^{i\theta_t}).
\nonumber
\end{align}
The output is
\begin{equation}
o_t=S_t^*q_t.
\end{equation}
In real implementations, the complex multiplication by $\lambda_{t,j}=\alpha_{t,j}e^{i\theta_{t,j}}$ is implemented by a $2\times2$ rotation-decay block. No FFT is required.

\paragraph{Stable parameterization.}
A simple implementation enforces the stability assumptions used later by
\begin{align}
\alpha_t&=\alpha_{\min}+(\alpha_{\max}-\alpha_{\min})\sigma(a_t),\qquad 0\le\alpha_{\min}\le\alpha_{\max}\le1,\label{eq:param-alpha}\\
\theta_t&=\theta_{\max}\tanh(b_t),\label{eq:param-theta}\\
\beta_t&=\sigma(c_t),\label{eq:param-beta}\\
k_t&=\tilde k_t/(\norm{\tilde k_t}_2+\varepsilon).\label{eq:param-key}
\end{align}
Then $\norm{\Lambda_t}_2\le1$, $0\le\beta_t\le1$, and $\norm{k_t}_2\le1$, so $0\le\beta_t\norm{k_t}_2^2\le1$.

\begin{definition}[Tied SFDA transition]
For a token $t$, define
\begin{equation}
A_t=(I-\beta_tk_tk_t^*)\Lambda_t,
\qquad
B_t=\beta_tk_tv_t^*.
\label{eq:AtBt}
\end{equation}
Then tied SFDA is the affine recurrence $S_t=A_tS_{t-1}+B_t$.
\end{definition}

\paragraph{Write-decoupled variant.}
For controlled memory constructions it is sometimes useful to separate the erase direction from the write direction:
\begin{equation}
S_t=(I-\eta_tk_tk_t^*)\Lambda_tS_{t-1}+\beta_tr_tv_t^*.
\label{eq:sfda-decoupled-main}
\end{equation}
The tied model is recovered by setting $r_t=k_t$ and $\eta_t=\beta_t$. The main architecture and kernel theorems are stated for tied SFDA.

\paragraph{Generalized tied-write affine template.}
The finite-state expressivity theorem in \cref{sec:finite-state-main,app:automata} uses a slightly more general proof template,
\begin{equation}
z_t=(I-\beta_{\sigma_t}k_{\sigma_t}k_{\sigma_t}^*)\widetilde A_{\sigma_t}z_{t-1}
+\beta_{\sigma_t}k_{\sigma_t}r_{\sigma_t}.
\label{eq:general-tied-template}
\end{equation}
Here $\widetilde A_\sigma$ is an unrestricted symbol-conditioned linear map. This template isolates the effect of tying erase and write through the same vector $k_\sigma$. It is an expressivity proof model, not the hardware kernel. The kernel-compatible SFDA family restricts $\widetilde A_\sigma$ to the phase-decay form $\Lambda_\sigma$, which is what enables the chunk-WY recursion.

\begin{table}[t]
\centering
\caption{Model hierarchy used in the paper. Each theorem explicitly states which row it applies to.}
\label{tab:model-hierarchy}
\begin{tabular}{p{0.25\linewidth}p{0.62\linewidth}}
\toprule
Model class & Linear transition family \\
\midrule
KDA & $(I-\beta kk^*)\diag(\alpha)$ \\
Tied SFDA & $(I-\beta kk^*)\Lambda$, with $\Lambda=\diag(\alpha\od e^{i\theta})$ or real block rotations \\
Kernel-compatible generalized SFDA & $\Lambda-ur^*$ with cheap structured $\Lambda$ \\
Generalized tied-write template & $(I-\beta kk^*)\widetilde A$, where $\widetilde A$ may be dense \\
\bottomrule
\end{tabular}
\end{table}

\begin{proposition}[KDA is a special case]
\label{prop:kda-special}
If $\theta_t=0$ for all $t$, then $\Lambda_t=\diag(\alpha_t)$ and \cref{eq:sfda} reduces exactly to KDA. Equivalently, in real coordinates, setting $U_t=I$ in \cref{eq:intro-sfda-real} recovers KDA.
\end{proposition}

\begin{proof}
Substitute $\theta_t=0$, so $e^{i\theta_t}=\one$ and $\Lambda_t=\diag(\alpha_t)$. The recurrence becomes \cref{eq:intro-kda}.
\end{proof}

\begin{proposition}[Complex phases are real block rotations]
\label{prop:realification-main}
Let $\lambda=\alpha e^{i\theta}$ and define
\begin{equation}
\mathcal R(\lambda)=\alpha
\begin{bmatrix}
\cos\theta&-\sin\theta\\
\sin\theta&\cos\theta
\end{bmatrix}.
\end{equation}
Multiplication by $\lambda$ on one complex coordinate is exactly multiplication by $\mathcal R(\lambda)$ on the corresponding two real coordinates. Hence
\begin{equation}
\diag(\lambda_1,\ldots,\lambda_m)
\quad\Longleftrightarrow\quad
\bigoplus_{j=1}^{m}\mathcal R(\lambda_j),
\end{equation}
and SFDA can be implemented with real tensors only.
\end{proposition}

\begin{proof}
This is the coordinate identity
\begin{equation}
\alpha e^{i\theta}(x+iy)=\alpha(x\cos\theta-y\sin\theta)+i\alpha(x\sin\theta+y\cos\theta).
\end{equation}
\end{proof}

\begin{proposition}[Bounded cyclic phase memory]
\label{prop:phase-memory}
Fix a two-dimensional real block, or one complex coordinate. Suppose $\beta_t=0$ and $\alpha_t=1$ on that block. If $\theta_t=2\pi a_t/m$ with $a_t\in\{0,\ldots,m-1\}$, then the hidden phase obeys
\begin{equation}
z_t=e^{2\pi i(\sum_{s\le t}a_s)/m}z_0,
\qquad
|z_t|=|z_0|.
\end{equation}
Thus SFDA exactly represents a cyclic counter $C_m$ with bounded norm.
\end{proposition}

\begin{proof}
The update on the selected coordinate is $z_t=e^{2\pi i a_t/m}z_{t-1}$. Multiplying phases adds exponents modulo $m$, and complex multiplication by a unit phase preserves norm.
\end{proof}

\paragraph{Single-step diagonal-plus-rank-one form.}
The SFDA transition matrix in \cref{eq:AtBt} can be written as
\begin{equation}
A_t=\Lambda_t-u_tw_t^*,
\qquad
u_t:=\beta_t k_t,
\qquad
w_t:=\Lambda_t^*k_t.
\label{eq:single-step-wy}
\end{equation}
Indeed, $u_tw_t^*=\beta_tk_tk_t^*\Lambda_t$. This is the invariant that makes the structured kernel possible: diagonal/phase decay plus one rank-one correction.

\section{Main Kernel Theorems}
\label{sec:theory-main}

The kernel theorem has two layers. First, phase-decay plus low-rank correction is closed under composition. Second, for the single-step SFDA form $A_t=\Lambda_t-u_tr_t^*$, there is an explicit left-to-right recursion for the factors. This turns the closure theorem into a constructive chunk algorithm.

\begin{theorem}[Block-WY closure]
\label{thm:wy-closure-main}
Let
\begin{equation}
A_1=\Gamma_1-Y_1M_1W_1^*,
\qquad
A_2=\Gamma_2-Y_2M_2W_2^*,
\end{equation}
where $\Gamma_i\in\F^{d\times d}$, $Y_i,W_i\in\F^{d\times r_i}$, and $M_i\in\F^{r_i\times r_i}$. Then
\begin{equation}
A_2A_1=\Gamma_{21}-Y_{21}M_{21}W_{21}^*,
\label{eq:wy-product-main}
\end{equation}
where
\begin{align}
\Gamma_{21} &= \Gamma_2\Gamma_1,\label{eq:Gamma21}\\
Y_{21} &= \begin{bmatrix}\Gamma_2Y_1 & Y_2\end{bmatrix},\label{eq:Y21}\\
W_{21} &= \begin{bmatrix}W_1 & \Gamma_1^*W_2\end{bmatrix},\label{eq:W21}\\
M_{21} &=
\begin{bmatrix}
M_1 & 0\\
-M_2W_2^*Y_1M_1 & M_2
\end{bmatrix}.
\label{eq:M21}
\end{align}
Consequently, structured transition tuples $(\Gamma,Y,W,M)$ form a representation of a matrix multiplication monoid, with rank growing additively under exact composition.
\end{theorem}

\begin{proof}
Expand $A_2A_1$ and substitute \cref{eq:Gamma21,eq:Y21,eq:W21,eq:M21}. The product $Y_{21}M_{21}W_{21}^*$ equals
\begin{align}
&\Gamma_2Y_1M_1W_1^*+Y_2M_2W_2^*\Gamma_1\nonumber\\
&\hspace{0.1in}-Y_2M_2W_2^*Y_1M_1W_1^*.
\end{align}
Therefore $\Gamma_{21}-Y_{21}M_{21}W_{21}^*$ is exactly $A_2A_1$. Associativity follows because each tuple represents an ordinary matrix.
\end{proof}

\begin{theorem}[Constructive left-to-right chunk-WY recursion]
\label{thm:constructive-wy}
Let $A_t=\Lambda_t-u_tr_t^*$ over $\F\in\{\R,\C\}$, where $\Lambda_t$ is diagonal or aligned block-phase and $r_t$ is the right WY vector. Define
\begin{equation}
P_t=A_tA_{t-1}\cdots A_1.
\end{equation}
Initialize $\Gamma_0=I$, $Y_0,W_0\in\F^{d\times0}$, and $M_0\in\F^{0\times0}$. For $t\ge1$, set
\begin{align}
\Gamma_t&=\Lambda_t\Gamma_{t-1},\label{eq:constructive-Gamma}\\
Y_t&=\begin{bmatrix}\Lambda_tY_{t-1}&u_t\end{bmatrix},\label{eq:constructive-Y}\\
W_t&=\begin{bmatrix}W_{t-1}&\Gamma_{t-1}^*r_t\end{bmatrix},\label{eq:constructive-W}\\
M_t&=
\begin{bmatrix}
M_{t-1}&0\\
-r_t^*Y_{t-1}M_{t-1}&1
\end{bmatrix}.
\label{eq:constructive-M}
\end{align}
Then
\begin{equation}
P_t=\Gamma_t-Y_tM_tW_t^*.
\label{eq:constructive-invariant}
\end{equation}
For a chunk of size $C$, $Y_C,W_C\in\F^{d\times C}$ and $M_C\in\F^{C\times C}$, so the correction rank is at most $C$.
\end{theorem}

\begin{proof}[Proof sketch]
The base case gives $P_1=\Lambda_1-u_1r_1^*$. For the induction step, expand $(\Lambda_t-u_tr_t^*)(\Gamma_{t-1}-Y_{t-1}M_{t-1}W_{t-1}^*)$. The lower row $-r_t^*Y_{t-1}M_{t-1}$ in \cref{eq:constructive-M} cancels the cross term $u_tr_t^*Y_{t-1}M_{t-1}W_{t-1}^*$. The full expansion is given in \cref{app:constructive-proof}.
\end{proof}

\begin{corollary}[SFDA admits exact constructive block-WY products]
\label{cor:sfda-wy}
For any chunk of $C$ tied SFDA transitions, set
\begin{equation}
\Lambda_t=\diag(\alpha_t\od e^{i\theta_t}),
\qquad
u_t=\beta_tk_t,
\qquad
r_t=\Lambda_t^*k_t.
\end{equation}
Then $A_t=(I-\beta_tk_tk_t^*)\Lambda_t=\Lambda_t-u_tr_t^*$ and \cref{thm:constructive-wy} gives
\begin{equation}
A_CA_{C-1}\cdots A_1=\Gamma_C-Y_CM_CW_C^*,
\label{eq:chunk-product}
\end{equation}
with $\rank(Y_CM_CW_C^*)\le C$.
\end{corollary}

\begin{proof}
The identity follows from $(\Lambda_t^*k_t)^*=k_t^*\Lambda_t$. Apply \cref{thm:constructive-wy}.
\end{proof}

\begin{theorem}[Exact affine chunk transfer]
\label{thm:affine-chunk-main}
Consider a chunk of SFDA transitions $S_t=A_tS_{t-1}+B_t$ for $t=1,\ldots,C$. Let
\begin{equation}
A_{C:1}:=A_CA_{C-1}\cdots A_1=\Gamma_C-Y_CM_CW_C^*.
\end{equation}
Define the chunk write summary
\begin{equation}
B_C^{\mathrm{chunk}}
:=\sum_{i=1}^{C} A_{C:i+1}B_i,
\qquad
A_{C:C+1}:=I.
\label{eq:chunk-write}
\end{equation}
Then the chunk maps its boundary state exactly by
\begin{equation}
S_{\mathrm{out}}
=(\Gamma_C-Y_CM_CW_C^*)S_{\mathrm{in}}+B_C^{\mathrm{chunk}}.
\label{eq:chunk-transfer}
\end{equation}
Moreover, $B_C^{\mathrm{chunk}}$ can be computed by running the same recurrence inside the chunk with zero initial state.
\end{theorem}

\begin{proof}
Unrolling $S_t=A_tS_{t-1}+B_t$ gives
\begin{equation}
S_C=A_{C:1}S_0+\sum_{i=1}^{C}A_{C:i+1}B_i.
\end{equation}
The first term is \cref{eq:chunk-product}; the second is \cref{eq:chunk-write}. If $\bar S_0=0$ and $\bar S_t=A_t\bar S_{t-1}+B_t$, then $\bar S_C$ is exactly the same sum.
\end{proof}

\begin{theorem}[Prefix outputs inside a chunk]
\label{thm:prefix-outputs}
For every prefix $r\le C$ of a chunk, define
\begin{equation}
P_r=A_rA_{r-1}\cdots A_1=\Gamma_r-Y_rM_rW_r^*
\end{equation}
using the constructive recursion in \cref{thm:constructive-wy}. Let
\begin{equation}
B_r^{\mathrm{chunk}}=\sum_{i=1}^{r}A_{r:i+1}B_i.
\end{equation}
Then the state after the $r$-th token in the chunk is
\begin{equation}
S_r=P_rS_0+B_r^{\mathrm{chunk}}.
\end{equation}
Consequently, the token output can be written exactly as
\begin{equation}
o_r=S_r^*q_r=S_0^*P_r^*q_r+(B_r^{\mathrm{chunk}})^*q_r.
\label{eq:prefix-output}
\end{equation}
Thus the same prefix factors used for the boundary transfer also generate all intra-chunk outputs.
\end{theorem}

\begin{proof}
Apply \cref{thm:affine-chunk-main} to the shortened chunk $1{:}r$. The output identity follows from the definition $o_r=S_r^*q_r$.
\end{proof}

\begin{remark}[What is not claimed]
The exact rank in \cref{thm:wy-closure-main} grows additively. We do not claim a fixed-rank exact representation over the entire sequence. The exact theorem is chunk-bounded: rank grows only up to the chosen chunk size $C$ and is then applied to boundary states.
\end{remark}

\subsection{Finite-state tied-write realization}
\label{sec:finite-state-main}

The constructive WY theorem is a kernel result. Separately, the following theorem resolves the finite-state version of the tied-write expressivity concern. The statement is intentionally elementary: a finite automaton is just a machine with finitely many states, and one-hot vectors turn its state update into matrix multiplication.

\begin{theorem}[Finite automata are exactly realized by the generalized tied-write template]
\label{thm:finite-tied-main}
Let $\mathcal A=(Q,\Sigma,\delta)$ be a deterministic finite automaton with $|Q|=N$. Consider the generalized tied-write affine template
\begin{equation}
z_t=(I-\beta_{\sigma_t}k_{\sigma_t}k_{\sigma_t}^*)\widetilde A_{\sigma_t}z_{t-1}
+\beta_{\sigma_t}k_{\sigma_t}r_{\sigma_t}
\end{equation}
in dimension $N$. Then there exist parameters such that $z_t$ is exactly the one-hot encoding of the automaton state after reading $\sigma_1,\ldots,\sigma_t$.
\end{theorem}

\begin{proof}
Index $Q=\{q_1,\ldots,q_N\}$ and encode $q_i$ as $e_i$. For each symbol $\sigma$, define the deterministic transition matrix $P_\sigma$ by $P_\sigma e_i=e_j$ whenever $\delta(q_i,\sigma)=q_j$. Choose $\widetilde A_\sigma=P_\sigma$ and $\beta_\sigma=0$. The tied-write update reduces to $z_t=P_{\sigma_t}z_{t-1}$, which is exactly the one-hot recurrence for the automaton. The detailed induction is given in \cref{app:finite-state-proof}.
\end{proof}

\begin{corollary}[Finite semidirect automata]
\label{cor:finite-semidir-main}
Every finite semidirect automaton with state space $Q=N\rtimes H$ is exactly realized by the generalized tied-write template in dimension $|N||H|$.
\end{corollary}

\begin{proof}
The product state space is finite, with $|Q|=|N||H|$. Apply \cref{thm:finite-tied-main} to the deterministic transition map induced by each input symbol.
\end{proof}

\begin{remark}[What this solves and what it does not solve]
\label{rem:finite-vs-compressed}
\Cref{thm:finite-tied-main} removes finite automata as an expressivity obstruction for the generalized tied-write template. It does not say that the kernel-compatible SFDA transition family $\Lambda_t-u_tw_t^*$ can represent every such transition in low dimension, nor that a compressed affine memory system can always be compiled into tied SFDA with polynomial overhead. Those are low-dimensional compression and kernel-design questions, not finite-state realizability questions.
\end{remark}

\subsection{Orbit-restricted tied-write routing}
\label{sec:orbit-routing-main}

The one-hot theorem above uses unrestricted symbol-conditioned matrices $\widetilde A_\sigma=P_\sigma$.  This is useful for finite-state expressivity, but it is not yet a kernel theorem for the structured SFDA family.  A sharper question is when the tied write direction can emulate a decoupled write direction while using only the control operators already available to the model.

Let $k$ be a tied key and let $\mathcal G$ be the structured control family, such as signed permutations, block rotations, or products of local rotations.  The reachable write directions from $k$ are its orbit
\begin{equation}
\Orb_{\mathcal G}(k)=\{Rk:R\in\mathcal G\}.
\end{equation}
If a desired decoupled write direction $w_\sigma$ lies in this orbit, then its write contribution can be routed from a tied write along $k$.  For a full affine update, one must also account for the erase term in the delta rule.  The exact theorem in \cref{app:tied_write_orbit} states the needed protected-subspace condition: after the base transition, the carried memory must have zero projection onto the tied key.  Under that condition, a tied write followed by the structured routing operator exactly realizes the decoupled write.

This result is intentionally restricted.  Full orthogonal control is transitive on the sphere but not kernel-compatible in general; fixed $2\times2$ block rotations preserve the chunk-WY structure but are not transitive on the full sphere.  Thus the practical SFDA design question becomes a routing-capacity problem: can a kernel-compatible control family provide enough write-direction routing to improve KDA-style memory without destroying the block-WY kernel?

\section{Compact Expressivity of Kernel-Compatible SFDA}
\label{sec:compact-expressivity}

The finite-state theorem in \cref{sec:finite-state-main} shows exact one-hot realizability for a generalized tied-write template.  That result is useful, but it does not answer the practical kernel question: which memory systems are represented \emph{compactly} by the actual structured transition family
\begin{equation}
A_t=\Lambda_t-u_tr_t^*,
\label{eq:compact-primitive}
\end{equation}
where $\Lambda_t$ is diagonal, block-phase, or another structured control operator that is cheap to multiply and apply?  This section gives a positive, scoped answer.  Kernel-compatible SFDA compactly represents three useful families: phase-group counters, register/reset memories, and bounded stacks under monomial or partial-permutation control.  It also gives the precise rank budget explaining why chunk-WY is the correct computational invariant.

\begin{definition}[Structured control family]
\label{def:structured-control}
A structured control family $\mathcal L\subseteq\F^{d\times d}$ is a set of matrices containing $I$ such that products and adjoints of its elements remain efficiently applicable.  The default SFDA choice is the complex diagonal phase-decay family
\begin{equation}
\mathcal L_{\rm phase}=\{\diag(\alpha\od e^{i\theta}):0\le \alpha_i\le1\}.
\end{equation}
The same chunk-WY algebra also applies when $\mathcal L$ is replaced by aligned real $2\times2$ rotation/reflection blocks, signed permutations, or partial permutations, provided products $\Gamma=\Lambda_C\cdots\Lambda_1$ and applications $\Gamma x$, $\Lambda Y$, and $\Gamma^*r$ remain cheap.
\end{definition}

\begin{definition}[$C$-compact SFDA realization]
\label{def:c-compact}
A symbolic memory system with transitions $T_\sigma$ is $C$-compactly realized by kernel-compatible SFDA in dimension $d$ if every symbol transition can be implemented by at most $C$ primitive affine steps
\begin{equation}
S^+=(\Lambda-u r^*)S+B,
\qquad \Lambda\in\mathcal L,
\end{equation}
using the same $d$-dimensional row state.  For tied SFDA, the primitive has $u=\beta k$, $r=\Lambda^*k$, and $B=\beta kv^*$.
\end{definition}

\subsection{Rank budget: the exact compact class inside a chunk}

\begin{theorem}[Kernel-compatible rank budget]
\label{thm:rank-budget}
Let $P_C=A_CA_{C-1}\cdots A_1$ be a product of $C$ primitive SFDA transitions $A_t=\Lambda_t-u_tr_t^*$ with $\Lambda_t\in\mathcal L$.  Then
\begin{equation}
P_C=\Gamma_C-Y_CM_CW_C^*,
\qquad
\Gamma_C=\Lambda_C\cdots\Lambda_1,
\end{equation}
where $Y_C,W_C\in\F^{d\times C}$ and $M_C\in\F^{C\times C}$.  Consequently
\begin{equation}
\rank(P_C-\Gamma_C)\le C.
\label{eq:rank-budget}
\end{equation}
Conversely, every transition generated by at most $C$ such primitive SFDA steps is represented exactly by these chunk-WY factors.
\end{theorem}

\begin{proof}
The factorization is exactly \cref{thm:constructive-wy}.  Since $Y_CM_CW_C^*$ factors through a $C$-dimensional middle space, its rank is at most $C$.  The converse is definitional: if the transition is generated by $C$ primitive steps, the constructive recursion produces its exact chunk representation.
\end{proof}

\begin{proposition}[No global fixed-rank exact invariant in general]
\label{prop:no-global-fixed-rank}
For generic rank-one delta transitions, the exact correction rank can grow linearly with sequence length up to $d$.  In particular, for any $r\le d$, there are $r$ primitive transitions with $\Lambda_i=I$ such that
\begin{equation}
\rank\!\left(I-A_r\cdots A_1\right)=r.
\end{equation}
\end{proposition}

\begin{proof}
Let $A_i=I-e_ie_i^*$ for $i=1,\ldots,r$.  These are primitive delta transitions with $\Lambda_i=I$, $u_i=e_i$, and $r_i=e_i$.  Since the coordinate projections commute,
\begin{equation}
A_r\cdots A_1=I-\sum_{i=1}^r e_ie_i^*.
\end{equation}
Therefore $I-A_r\cdots A_1=\sum_{i=1}^r e_ie_i^*$ has rank $r$.
\end{proof}

\paragraph{Consequence.}
\Cref{prop:no-global-fixed-rank} explains why the paper uses chunk-bounded exactness.  A fixed-rank global summary is impossible in general.  The exact invariant is local: let the rank grow to $C$ inside a chunk, apply the chunk transfer to the boundary state, and restart the rank budget in the next chunk.

\subsection{Useful compactly representable memory systems}

\begin{theorem}[Direct products of cyclic counters]
\label{thm:cyclic-products-compact}
Let
\begin{equation}
G=C_{m_1}\times\cdots\times C_{m_\ell}
\end{equation}
be a product of cyclic groups.  Suppose each input symbol $\sigma$ carries an increment vector $a(\sigma)=(a_1(\sigma),\ldots,a_\ell(\sigma))$ with $a_j(\sigma)\in\mathbb Z_{m_j}$.  Then phase-only SFDA exactly realizes the group state in complex dimension $\ell$ with chunk size $C=1$ by setting
\begin{equation}
\beta_\sigma=0,
\qquad
\Lambda_\sigma=\diag\!\left(e^{2\pi i a_1(\sigma)/m_1},\ldots,e^{2\pi i a_\ell(\sigma)/m_\ell}\right).
\end{equation}
The state norm is preserved, and the real implementation uses dimension $2\ell$.
\end{theorem}

\begin{proof}
Let $z_0=\one\in\C^\ell$.  With $\beta_\sigma=0$, the recurrence is $z_t=\Lambda_{\sigma_t}z_{t-1}$.  Therefore
\begin{equation}
[z_t]_j=\exp\!\left(2\pi i\sum_{s\le t}a_j(\sigma_s)/m_j\right),
\end{equation}
which is exactly the character embedding of the current element of $G$.  Each coordinate has unit modulus, so the norm is bounded.  The complex-to-real equivalence is \cref{lem:real-complex}.
\end{proof}

\begin{corollary}[Compactness over one-hot state]
The realization in \cref{thm:cyclic-products-compact} represents a state space of size $\prod_{j=1}^\ell m_j$ using $2\ell$ real coordinates, before the finite readout.  Thus phase SFDA gives exponentially smaller state dimension than one-hot encoding for product counters.
\end{corollary}

\begin{theorem}[Dihedral phase-orientation memory]
\label{thm:dihedral-compact}
Assume the structured control family contains the planar rotation
\begin{equation}
R=\begin{bmatrix}\cos(2\pi/m)&-\sin(2\pi/m)\\ \sin(2\pi/m)&\cos(2\pi/m)\end{bmatrix}
\end{equation}
and the reflection
\begin{equation}
F=\begin{bmatrix}1&0\\0&-1\end{bmatrix}.
\end{equation}
Then the dihedral group $D_m=\langle R,F:F^2=I,\ F R F=R^{-1}\rangle$ is exactly realized by kernel-compatible SFDA with $\beta=0$, row dimension $d=2$, value width $d_v=2$, and chunk size $C=1$, by maintaining the matrix state $S_t=\rho(g_t)\in\R^{2\times2}$.
\end{theorem}

\begin{proof}
Let the input symbol choose either $R$, $F$, or a product generator in $D_m$, and set $\Lambda_\sigma=\rho(\sigma)$, where $\rho$ is the faithful planar representation generated by $R$ and $F$.  With $\beta=0$, the update is $S_t=\Lambda_{\sigma_t}S_{t-1}$.  Starting from $S_0=I$, induction gives $S_t=\rho(\sigma_t\cdots\sigma_1)$.  The matrices $R$ and $F$ satisfy the defining relations of $D_m$, so this exactly tracks the dihedral state.  Orthogonality gives bounded norm.
\end{proof}

\begin{remark}[Pure phase vs. reflection]
The default complex-diagonal phase family represents the rotation subgroup $C_m$.  Exact dihedral tracking needs a reflection or conjugation operator.  This remains kernel-compatible when the control family is expanded to aligned $2\times2$ orthogonal blocks, because the WY recursion only requires cheap multiplication and application of the structured controls.
\end{remark}

\begin{theorem}[Register and reset memories]
\label{thm:register-compact}
Consider $d$ independent memory registers, each storing a value vector in $\R^{d_v}$.  The operation ``write value $a$ into register $i$ and leave all other registers unchanged'' is exactly represented by tied SFDA with chunk size $C=1$:
\begin{equation}
S^+=(I-e_ie_i^*)S+e_i a^*.
\end{equation}
Equivalently, choose $\Lambda=I$, $\beta=1$, $k=e_i$, and $v=a$.
\end{theorem}

\begin{proof}
Substituting $\Lambda=I$, $\beta=1$, and $k=e_i$ into tied SFDA gives
\begin{equation}
S^+=(I-e_ie_i^*)S+e_i a^*.
\end{equation}
The projection $I-e_ie_i^*$ zeros exactly row $i$ and leaves all other rows unchanged; the write term replaces row $i$ by $a^*$.
\end{proof}

\begin{corollary}[Flip-flops and last-write-wins memories]
Finite flip-flop memories, reset automata, and last-write-wins associative slots are compactly represented by KDA/SFDA with one row per memory slot and one value channel per symbol feature, rather than one coordinate per global finite state.
\end{corollary}

\begin{theorem}[Bounded stacks under partial-permutation control]
\label{thm:bounded-stack-compact}
Suppose the structured control family contains the zero-fill shift matrices $Z_{\downarrow},Z_{\uparrow}\in\R^{h\times h}$ defined by
\begin{equation}
[Z_{\downarrow}S]_{i+1,:}=S_{i,:}\ (i<h),\quad [Z_{\downarrow}S]_{1,:}=0,
\end{equation}
with $Z_{\uparrow}$ shifting rows upward and zero-filling the bottom row.  Then a bounded stack of depth $h$ over value embeddings $a\in\R^{d_v}$ is exactly represented by kernel-compatible SFDA in row dimension $h$:
\begin{align}
\mathrm{PUSH}(a):\quad&S^+=Z_{\downarrow}S+e_1a^*,\label{eq:push}\\
\mathrm{POP}:\quad&S^+=Z_{\uparrow}S.\label{eq:pop}
\end{align}
The push operation is tied-SFDA-compatible because $e_1^*Z_{\downarrow}=0$.
\end{theorem}

\begin{proof}
For push, set $\Lambda=Z_{\downarrow}$, $\beta=1$, $k=e_1$, and $v=a$.  The tied transition gives
\begin{equation}
S^+=(I-e_1e_1^*)Z_{\downarrow}S+e_1a^*.
\end{equation}
Since the first row of $Z_{\downarrow}S$ is zero, $(I-e_1e_1^*)Z_{\downarrow}S=Z_{\downarrow}S$, proving \cref{eq:push}.  For pop, set $\Lambda=Z_{\uparrow}$ and $\beta=0$, giving \cref{eq:pop}.
\end{proof}

\begin{remark}[Scope of the stack theorem]
The bounded-stack theorem uses partial-permutation controls, not the default phase-only controls.  It is included because the constructive WY theorem is valid for any structured control family whose products and applications are cheap.  If the architecture is restricted strictly to complex diagonal phases, then cyclic/product counters and oscillatory filters are the default compact automata family, while stack shifts require an additional monomial or partial-permutation control option.
\end{remark}

\subsection{What this section proves}

The compact expressivity picture is now precise:
\begin{itemize}[leftmargin=*]
    \item phase-only SFDA compactly represents finite abelian phase counters and oscillatory diagonal SSM-style filters;
    \item adding aligned reflections gives compact dihedral and orientation memory;
    \item the rank-one delta correction gives exact register/reset memories;
    \item adding cheap partial permutations gives bounded stacks and shift-register memories;
    \item arbitrary finite automata remain exactly realizable by one-hot lifting in the generalized template, but not necessarily compactly by $\Lambda-u r^*$ transitions.
\end{itemize}
Thus the remaining empirical question is not whether SFDA can represent any finite state machine in principle, but whether these compact structured memories improve extrapolation and the KDA quality--efficiency frontier.

\section{Chunk Algorithm, Cost, and Stability}
\label{sec:kernel}

The previous section proves exact constructive factorization. This section turns it into a kernel theorem. The central rule is simple: \emph{do not allow the WY rank to grow across the whole sequence}. Let it grow only inside fixed chunks of length $C$.

\begin{algorithm}[t]
\caption{Constructive SFDA-WY chunk summary}
\label{alg:chunk-summary}
\begin{algorithmic}[1]
\Require Chunk inputs $\{k_t,a_t,\alpha_t,\theta_t,\beta_t\}_{t=1}^C$, where $a_t$ is the attention value vector.
\State $\Gamma\gets I$; $Y,W\gets$ empty $d\times0$ matrices; $M\gets$ empty $0\times0$ matrix.
\For{$t=1$ to $C$}
    \State $\Lambda_t\gets\diag(\alpha_t\od e^{i\theta_t})$.
    \State $u_t\gets\beta_t k_t$ and $r_t\gets\Lambda_t^*k_t$ so $A_t=\Lambda_t-u_tr_t^*$.
    \State $Y_{\rm old}\gets Y$, $W_{\rm old}\gets W$, $M_{\rm old}\gets M$, $\Gamma_{\rm old}\gets\Gamma$.
    \State $\Gamma\gets\Lambda_t\Gamma_{\rm old}$.
    \State $Y\gets[\Lambda_tY_{\rm old},\ u_t]$.
    \State $W\gets[W_{\rm old},\ \Gamma_{\rm old}^*r_t]$.
    \State $M\gets\begin{bmatrix}M_{\rm old}&0\\-r_t^*Y_{\rm old}M_{\rm old}&1\end{bmatrix}$.
\EndFor
\State Compute $B_C^{\rm chunk}$ by zero-state recurrence $\bar S_0=0$, $\bar S_t=A_t\bar S_{t-1}+\beta_tk_ta_t^*$.
\State \Return $\transfer=(\Gamma,Y,W,M,B_C^{\rm chunk})$.
\end{algorithmic}
\end{algorithm}

\paragraph{Applying a chunk transfer.}
Given $\transfer=(\Gamma,Y,W,M,B)$, the boundary update is
\begin{equation}
S_{\mathrm{out}}=\Gamma S_{\mathrm{in}}-Y\bigl(M(W^*S_{\mathrm{in}})\bigr)+B.
\label{eq:apply-transfer}
\end{equation}
If $Y,W\in\F^{d\times C}$ and $S\in\F^{d\times d_v}$, this costs
\begin{equation}
\bigO(dd_v)+\bigO(dCd_v)+\bigO(C^2d_v)+\bigO(dCd_v),
\label{eq:apply-cost}
\end{equation}
where $\Gamma S$ is cheap because $\Gamma$ is diagonal or block diagonal. For fixed $C\in\{64,128,256\}$, the transfer application is linear in the state dimension up to a chunk-dependent factor.

\begin{theorem}[Chunk-bounded exactness]
\label{thm:chunk-bounded}
For a sequence of length $T$ partitioned into chunks of size $C$, \cref{alg:chunk-summary} and \cref{eq:apply-transfer} compute exactly the same boundary states as the recurrent SFDA update, while maintaining low-rank factors of rank at most $C$ inside every chunk.
\end{theorem}

\begin{proof}
The constructive recursion gives $A_{C:1}=\Gamma-YMW^*$ exactly by \cref{thm:constructive-wy}. The zero-state recurrence gives the affine write summary by \cref{thm:affine-chunk-main}. Applying summaries sequentially gives $S_{j+1}=A_jS_j+B_j$, which is the original recurrence with parentheses grouped by chunks. The rank bound follows because a chunk contains $C$ single-step rank-one corrections.
\end{proof}

\paragraph{Inter-chunk parallelism.}
The boundary-state scan is exact and simple. It can be implemented recurrently over $T/C$ chunks; this is often acceptable because $T/C\ll T$. A fully parallel inter-chunk scan is possible using the affine law \cref{eq:affine-law}, but exact composition doubles the low-rank size. Therefore there are three implementation choices:
\begin{description}[leftmargin=*]
    \item[Boundary-state scan.] Do not compose $(Y,W,M)$ across chunks. Apply each chunk transfer to the incoming boundary state. This is the recommended first kernel.
    \item[Superchunks.] Let rank grow inside a superchunk of size $C_s=mC$, then apply the resulting transfer to a boundary state. This improves parallelism while keeping rank bounded by $C_s$.
    \item[Approximate recompression.] Compose chunk transfers and compress $YMW^*$ back to rank $C$ using QR/SVD or randomized low-rank approximation. This gives more parallelism but is approximate, so it should not be the first exact theorem claim.
\end{description}

\begin{theorem}[Reference complexity]
\label{thm:complexity}
For one chunk of size $C$, state width $d$, and value width $d_v$, the constructive SFDA-WY summary can be stored using
\begin{equation}
\bigO(d)+\bigO(dC)+\bigO(C^2)+\bigO(dd_v)
\end{equation}
scalars for $(\Gamma,Y,W,M,B)$. The naive explicit construction of $(\Gamma,Y,W,M)$ costs $\bigO(dC^2+C^3)$ arithmetic operations per chunk, the zero-state write recurrence costs $\bigO(Cdd_v)$, and applying the chunk transfer to a boundary state costs
\begin{equation}
\bigO(dd_v+dCd_v+C^2d_v).
\end{equation}
For fixed $C$, boundary application is linear in $d$ and $d_v$ up to constants depending on $C$.
\end{theorem}

\begin{proof}
The storage terms are for a diagonal/block-diagonal $\Gamma$, two $d\times C$ factors, a $C\times C$ triangular factor, and the additive state-shaped term $B$. In the naive recursion, multiplying $\Lambda_tY_{t-1}$ and appending $\Gamma_{t-1}^*r_t$ over all $t\le C$ costs $\bigO(dC^2)$; forming rows $r_t^*Y_{t-1}M_{t-1}$ costs $\bigO(dC^2+C^3)$ over the chunk. The zero-state recurrence applies one phase-decay and one rank-one correction to a $d\times d_v$ state for $C$ steps, costing $\bigO(Cdd_v)$. Finally, \cref{eq:apply-transfer} consists of $W^*S$, multiplication by $M$, multiplication by $Y$, and the cheap phase-decay term, giving the stated bound.
\end{proof}

\begin{proposition}[Basic numerical stability]
\label{prop:stability}
Assume $\norm{\Lambda_t}_2\le1$ and $0\le\beta_t\norm{k_t}_2^2\le2$ for all $t$. Then $\norm{A_t}_2\le1$ for $A_t=(I-\beta_tk_tk_t^*)\Lambda_t$. Consequently, if $\norm{B_t}_F\le b_t$, then
\begin{equation}
\norm{S_t}_F\le \norm{S_0}_F+\sum_{i=1}^t b_i.
\end{equation}
If additionally $\norm{\Lambda_t}_2\le\rho<1$ and $\norm{I-\beta_tk_tk_t^*}_2\le1$, then
\begin{equation}
\norm{S_t}_F\le \rho^t\norm{S_0}_F+\sum_{i=1}^t \rho^{t-i}b_i.
\end{equation}
\end{proposition}

\begin{proof}
The matrix $I-\beta_tk_tk_t^*$ has eigenvalue $1-\beta_t\norm{k_t}_2^2$ on the span of $k_t$ and eigenvalue $1$ on its orthogonal complement. The assumption $0\le\beta_t\norm{k_t}_2^2\le2$ gives spectral norm at most $1$. Therefore $\norm{A_t}_2\le\norm{I-\beta_tk_tk_t^*}_2\norm{\Lambda_t}_2\le1$. The Frobenius bounds follow from $\norm{A_tS}_{F}\le\norm{A_t}_2\norm{S}_F$ and induction.
\end{proof}

\begin{proposition}[Conditional conditioning of triangular WY factors]
\label{prop:triangular-conditioning}
Suppose the chunk recursion can be written as
\begin{equation}
M=(I+L)^{-1}D_\beta,
\end{equation}
where $L$ is strictly lower triangular and $D_\beta$ is diagonal with $\norm{D_\beta}_2\le1$. If $\norm{L}_2\le\eta<1$, then
\begin{equation}
\norm{M}_2\le \frac{1}{1-\eta}.
\end{equation}
\end{proposition}

\begin{proof}
Since $\norm{L}_2<1$, the Neumann series gives $(I+L)^{-1}=\sum_{j\ge0}(-L)^j$ and hence $\norm{(I+L)^{-1}}_2\le(1-\eta)^{-1}$. Multiplying by $D_\beta$ cannot increase the norm by more than $\norm{D_\beta}_2\le1$.
\end{proof}

\paragraph{Training-time enforcement.}
The parameterization in \cref{eq:param-alpha,eq:param-theta,eq:param-beta,eq:param-key} enforces $\norm{\Lambda_t}_2\le1$, $0\le\beta_t\le1$, and $\norm{k_t}_2\le1$.  The conditioning condition in \cref{prop:triangular-conditioning} is not automatic; it should be monitored empirically through $\norm{L}_2$, triangular-solve residuals, or finite-precision agreement between recurrent and chunk paths.

\paragraph{Systems target.}
The theorem above is a reference bound, not the final fused-kernel claim. The systems target is a Triton/CUDA implementation that avoids dense $d\times d$ matrices, uses matrix multiplies for $W^*S$, $M(\cdot)$, and $Y(\cdot)$, and empirically approaches KDA-like $\bigO(TdC)$ behavior with small constants. That speed claim must be benchmarked; algebra alone does not prove it.


\section{Basic Experiments}
\label{sec:basic-experiments}

We report a small set of controlled experiments designed to test the algebraic
mechanism isolated by the theory. These experiments are intentionally modest:
they verify the exact chunk algebra numerically and test whether phase control
helps on cyclic state tracking. They do not claim large-scale language-model
improvements, hybrid-ratio gains, or fused-kernel speedups; those remain the
next empirical stage.

\paragraph{Experiment A: constructed cyclic memory.}
We first test a training-free mod-$5$ counter.  The target state is
\[
    c_t = \sum_{i=1}^t a_i \pmod 5,
\]
and SFDA stores the counter in a two-dimensional phase block
\[
    z_t = R(2\pi a_t/5) z_{t-1}.
\]
The decoder is calibrated at length $128$ and then evaluated at longer lengths.
This directly tests Proposition~\ref{prop:phase-memory}: a unit-modulus phase
orbit can store a cyclic counter without norm drift.  As a controlled baseline,
we compare against the best real-diagonal decay model in the same low-dimensional
counter setting.  Real diagonal decay has no nontrivial bounded period-$5$ orbit,
so it should collapse toward chance as the length grows.

\paragraph{Experiment B: learned short-to-long counter.}
We next train a minimal SFDA counter head at length $48$ and evaluate on longer
lengths.  The baseline is the same architecture with the phase disabled,
\(\theta_t\equiv 0\), corresponding to the real-decay KDA special case in this
controlled setup. This experiment asks whether the learnable phase parameter can
be found from data, and whether the learned phase extrapolates beyond the
training length.

\begin{figure}[t]
    \centering
    \begin{minipage}{0.49\textwidth}
        \centering
        \includegraphics[width=\linewidth]{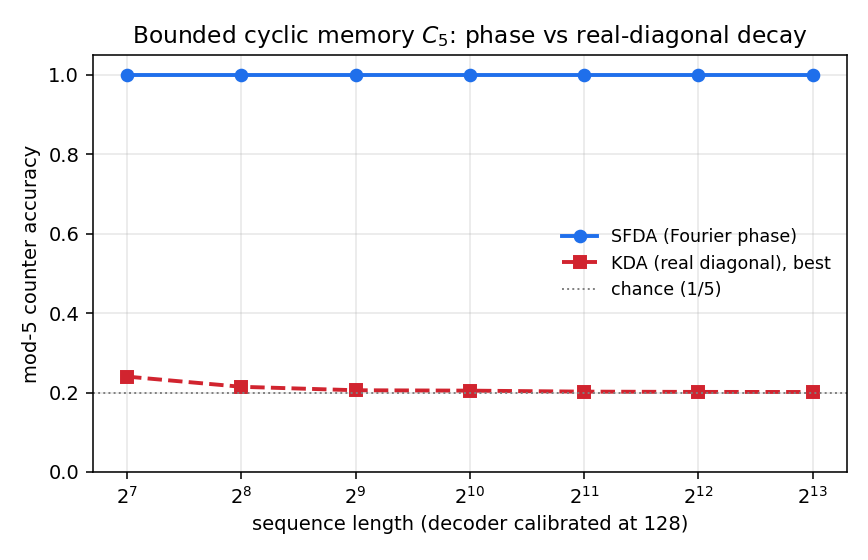}
        \vspace{-0.75em}
        \centerline{\small (a) Constructed phase counter}
    \end{minipage}
    \hfill
    \begin{minipage}{0.49\textwidth}
        \centering
        \includegraphics[width=\linewidth]{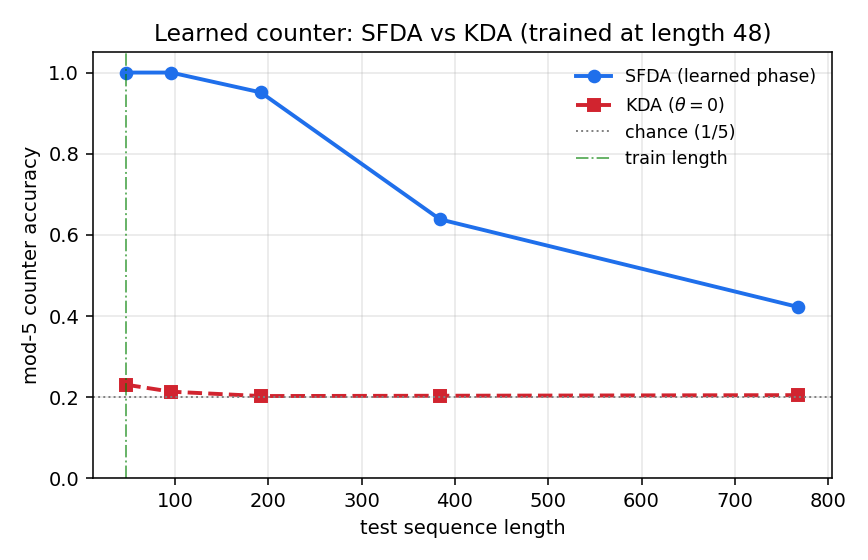}
        \vspace{-0.75em}
        \centerline{\small (b) Learned counter, train length $48$}
    \end{minipage}
    \caption{\textbf{Toy cyclic-memory experiments.}
    (a) In a constructed mod-$5$ counter, SFDA's Fourier phase block tracks the
    counter exactly from length $2^7$ to $2^{13}$, while the best real-diagonal
    decay baseline stays near chance.  (b) In a learned mod-$5$ counter trained
    only at length $48$, SFDA learns a usable phase update and extrapolates
    beyond the training length, while the phase-disabled KDA baseline remains
    near chance in this controlled setting.}
    \label{fig:toy-cyclic-memory}
\end{figure}

\begin{table}[t]
\centering
\caption{Constructed mod-$5$ counter: representational separation.}
\label{tab:constructed-counter}
\vspace{-0.25em}
\begin{tabular}{rcc}
\toprule
Length & SFDA phase & Best real diagonal \\
\midrule
128  & 1.000 & 0.240 \\
1024 & 1.000 & 0.205 \\
8192 & 1.000 & 0.201 \\
\bottomrule
\end{tabular}
\vspace{-0.5em}
\end{table}

\begin{table}[t]
\centering
\caption{Learned mod-$5$ counter trained at length $48$.}
\label{tab:learned-counter}
\vspace{-0.25em}
\begin{tabular}{rcc}
\toprule
Test length & KDA, $\theta=0$ & SFDA \\
\midrule
48  & 0.231 & 1.000 \\
96  & 0.214 & 1.000 \\
192 & 0.203 & 0.951 \\
384 & 0.204 & 0.638 \\
768 & 0.205 & 0.422 \\
\bottomrule
\end{tabular}
\vspace{-0.5em}
\end{table}

\paragraph{Results.}
Figure~\ref{fig:toy-cyclic-memory} and
Tables~\ref{tab:constructed-counter}--\ref{tab:learned-counter} show the
expected separation.  In the constructed setting, phase memory is exact at all
tested lengths, whereas the real-diagonal baseline approaches chance level
\(1/5\).  In the learned setting, SFDA solves the train-length counter and
extrapolates roughly four times beyond the training length before phase error
accumulates.  The phase-disabled baseline stays near chance.  The gap between
Experiment~A and Experiment~B suggests that the SFDA parameterization has the
right representation, while training and phase calibration determine how far the
learned model extrapolates.

\paragraph{Numerical theorem checks.}
We also verify the algebra used by the kernel construction.  The block-WY
closure theorem, the constructive chunk-WY product, and the affine chunk transfer
all match dense recurrent ground truth to numerical precision; see
Appendix~\ref{app:toy-experiment-details}.  These checks are not a substitute
for fused-kernel benchmarking, but they guard against algebraic implementation
mistakes before moving to a Triton/CUDA kernel.


\section{Experiments}
\label{sec:experiments}

The theory in Sections~\ref{sec:theory-main}--\ref{sec:kernel} makes two
kinds of falsifiable claims: (i) exact algebraic identities, including the
constructive chunk-WY recursion, the affine chunk transfer, the KDA reduction,
finite-state one-hot realizability, and the stability bound; and (ii) a
mechanistic claim that phase control implements bounded cyclic memory that a
real-only decay mechanism does not provide. We verify (i) numerically and test
(ii) on small synthetic state-tracking tasks. These are proof-of-concept
experiments; language-modeling, fused-kernel, and long-context claims remain
future work, as discussed in Section~\ref{sec:limitations} and
Appendix~\ref{app:experiment-protocol}.

\subsection{Numerical verification of the theorems}
\label{sec:exp-verify}

We implement the SFDA recurrence, the dense single-step transition
\[
A_t=(I-\beta_t k_t k_t^\ast)\Lambda_t,
\]
and the constructive chunk-WY recursion of
Theorem~\ref{thm:constructive-wy} in complex coordinates, and check each theorem
against a brute-force ground truth on random inputs. Table~\ref{tab:verify}
reports the residuals; all identities hold to numerical precision.
\footnote{The cyclic phase check is reported in double precision. In single
precision, the phase counter accumulates roughly \(10^{-4}\) round-off over
\(T=5000\) steps, consistent with the theorem being exact and the observed error
being numerical.}

\begin{table}[t]
\centering
\caption{Numerical verification of SFDA's exact claims. Residual is the relevant
operator or state error against a brute-force reference on random inputs.}
\label{tab:verify}
\small
\begin{tabular}{@{}llr@{}}
\toprule
Claim & Reference & Residual \\
\midrule
Chunk-WY recursion
 & Thm.~\ref{thm:constructive-wy}
 & \(5.7{\times}10^{-10}\) \\
Correction rank \(\le C\)
 & Cor.~\ref{cor:sfda-wy}
 & \(=C\) \\
KDA reduction at \(\theta=0\)
 & Prop.~\ref{prop:kda-special}
 & \(0\) \\
Cyclic phase memory, \(T=5000\)
 & Prop.~\ref{prop:phase-memory}
 & \(5.2{\times}10^{-13}\) \\
Affine chunk transfer
 & Thm.~\ref{thm:affine-chunk-main}
 & \(9.5{\times}10^{-8}\) \\
Generalized tied-write DFA realization
 & Thm.~\ref{thm:finite-tied-main}
 & \(0\) \\
Stability bound \(\|A_t\|_2\le 1\)
 & Prop.~\ref{prop:stability}
 & holds over \(2000\) samples \\
\bottomrule
\end{tabular}
\end{table}

\subsection{State tracking and length extrapolation}
\label{sec:exp-track}

\paragraph{Setup.}
We compare tied SFDA against its \(\theta=0\) reduction, which is the KDA
special case from Proposition~\ref{prop:kda-special}. The comparison uses an
equal-capacity design: identical architecture, identical projections, and
identical parameter count; the two models differ only in whether the phase
\(\theta_t\) is used or forced to zero. Each model is a single-layer,
single-head recurrence with state width \(d=16\) and a linear readout
(\(\sim 8.3\mathrm{k}\) parameters total). We train with AdamW using a cosine
schedule, keep the best checkpoint by in-distribution validation accuracy, and
report accuracy on the final quarter of positions, which are the positions that
most require persistent state. All models train on length
\(T_{\mathrm{train}}=32\) and are evaluated on
\(T\in\{32,64,128,256\}\), i.e., up to \(8\times\) extrapolation. Results are
means over two seeds.

\paragraph{Tasks.}
We evaluate two small state-tracking tasks. First, the cyclic counter \(C_3\)
asks the model to output the running sum modulo \(3\) of a stream of
increments. Second, the reset counter \(C_3\) is the same task with an additional
reset symbol that returns the state to \(0\). The reset operation is an
idempotent memory transition, matching the register/reset family in
Theorem~\ref{thm:register-compact}. Chance accuracy is \(1/3\).

\paragraph{Results.}
Table~\ref{tab:track} and Figure~\ref{fig:track} show a clean separation. On
the pure cyclic counter, the \(\theta=0\) KDA reduction remains near chance at
all tested lengths, including in-distribution. SFDA solves the task
in-distribution and degrades only at longer extrapolation lengths, where learned
phase calibration error accumulates. With a reset symbol bounding the effective
horizon, SFDA extrapolates essentially perfectly to \(8\times\) the training
length, while KDA plateaus around \(0.68\). This matches
Proposition~\ref{prop:phase-memory}: phase is the mechanism that gives SFDA a
bounded cyclic orbit.

\begin{table}[t]
\centering
\caption{State-tracking accuracy on last-quarter positions. Models are trained
at length \(32\), and results are means over \(2\) seeds. SFDA and KDA have
identical capacity; KDA is SFDA with \(\theta=0\). Chance is \(0.333\).}
\label{tab:track}
\small
\begin{tabular}{@{}llcccc@{}}
\toprule
Task & Model & \(T=32\) & \(T=64\) & \(T=128\) & \(T=256\) \\
\midrule
\multirow{2}{*}{Cyclic \(C_3\)}
 & SFDA & \textbf{1.000} & \textbf{0.987} & \textbf{0.491} & 0.341 \\
 & KDA  & 0.344 & 0.330 & 0.337 & 0.335 \\
\midrule
\multirow{2}{*}{Reset \(C_3\)}
 & SFDA & \textbf{1.000} & \textbf{1.000} & \textbf{1.000} & \textbf{1.000} \\
 & KDA  & 0.686 & 0.682 & 0.683 & 0.680 \\
\bottomrule
\end{tabular}
\end{table}

\begin{figure}[t]
\centering
\includegraphics[width=\linewidth]{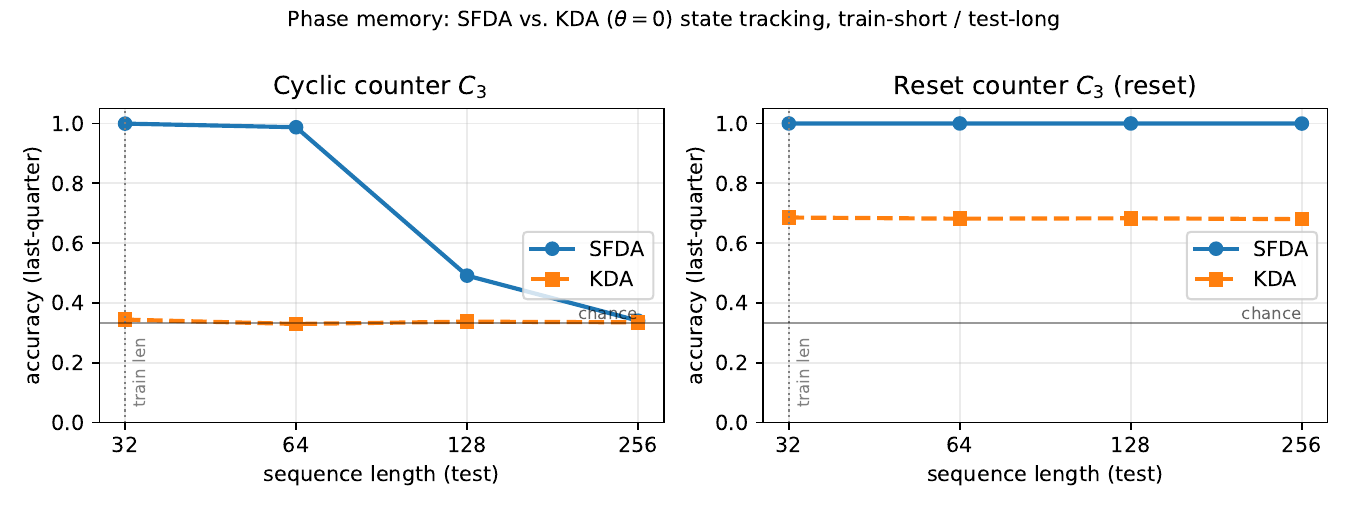}
\caption{Phase memory enables modular state tracking. SFDA, shown with solid
curves, is compared to its \(\theta=0\) KDA reduction, shown with dashed curves,
on cyclic and reset counters. Models are trained at length \(32\), indicated by
the dotted vertical line, and tested out to length \(256\). In this controlled
setting, KDA remains near chance on the cyclic counter, while SFDA solves the
task in-distribution and extrapolates exactly on the reset counter.}
\label{fig:track}
\end{figure}

\paragraph{Limitations of these experiments.}
These are small-scale, mechanism-isolating tests, not the full empirical program
needed for a systems-level ICML submission. The dihedral task \(D_3\), a
reversible counter whose direction is flipped by a control symbol, remains hard
at this scale: SFDA reaches only approximately \(0.47\) in-distribution, while
KDA stays near chance. This indicates that learning conditional
direction-reversal, while representable by the structured control theory in
Theorem~\ref{thm:dihedral-compact}, may require more capacity or training than
used here. We make no language-modeling, throughput, or long-context claims in
this section. In particular, the hybrid-ratio target---SFDA at \(7{:}1\) or
\(15{:}1\) versus KDA at \(3{:}1\)---and the fused Triton/CUDA kernel are left
to future work; see Appendix~\ref{app:experiment-protocol} and
Section~\ref{sec:limitations}.
\section{Related Work}
\label{sec:related-work}

\paragraph{Efficient attention and fast-weight memory.}
The Transformer architecture introduced softmax self-attention as a general sequence modeling primitive \citep{vaswani2017attention}. Linear attention rewrites attention as an associative recurrent memory, enabling constant-size autoregressive state and linear-time sequence processing \citep{katharopoulos2020linear}.  This view is closely related to fast-weight memory: the recurrent state stores transient key--value associations and is updated online by Hebbian or delta-style rules \citep{schlag2021fastweights}.  Subsequent work improves the basic linear-attention state with gating, decay, and hardware-aware chunkwise parallelism \citep{hua2022linear,sun2023retnet,yang2024gla,yang2025gdn}.  Kimi Linear introduces KDA, a channel-wise gated delta-rule layer with a specialized DPLR/WY-style chunkwise algorithm and a 3:1 KDA-to-global-attention hybrid \citep{kimi2025linear}.  SFDA is complementary: rather than replacing KDA's delta memory, it adds a phase-control operator and proves that the resulting transition still admits an exact chunk-local WY invariant.

\paragraph{State-space models, DPLR structure, and phase dynamics.}
Structured state-space models use diagonal or diagonal-plus-low-rank operators to model long sequences efficiently \citep{gu2022s4}.  Mamba and Mamba2 make state-space dynamics input-dependent and expose a duality between selective SSMs and attention-like computation \citep{gu2023mamba,dao2024mamba2}.  SFDA differs from static complex diagonal SSMs because the transition contains an input-dependent delta correction, and differs from unrestricted DPLR because the low-rank term is tied to the delta rule, preserving a chunk-WY kernel.  The learned phase component is motivated by the empirical and theoretical importance of oscillatory state dynamics and negative/complex modes for state tracking \citep{grazzi2025negative}.

\paragraph{Chunk-WY kernels and triangular transforms.}
The algebraic core of KDA-like kernels is a compact representation of products of rank-one corrections.  Classical WY representations compact products of Householder transformations \citep{bischof1987wy}, and later UT-style triangular transforms improve accumulation efficiency \citep{joffrain2006ut}.  KDA adapts this philosophy to diagonal decay plus rank-one delta updates and benchmarks a faster specialized kernel than general DPLR \citep{kimi2025linear}.  Parallelizing delta-rule linear transformers over sequence length has also been studied directly in the linear-attention kernel literature \citep{yang2024delta,yang2024fla}.  Our constructive theorem gives the corresponding invariant for phase-decay transitions $A_t=\Lambda_t-u_tr_t^*$.

\paragraph{Position, phases, and rotations.}
Rotary positional embeddings encode relative position through products of two-dimensional rotations applied to query/key features \citep{su2024roformer,barbero2025rope}.  Related Fourier positional methods study periodic extensions for length generalization \citep{hua2024fourierposition}.  SFDA uses rotations differently: phases act on the recurrent memory transition itself, so the phase variable is a state-transition mechanism rather than only a positional feature map.

\paragraph{Automata, shortcuts, and algebraic memory.}
Transformers can simulate finite automata by composing transition maps in shallow parallel circuits, and solvable semiautomata admit particularly short Krohn--Rhodes-style shortcuts \citep{liu2023shortcuts,krohn1965algebraic,barrington1988finite}.  The earlier semidirect/Fourier automata preprint proposed using algebraic memory/control decompositions and Fourier phase representations for automaton simulation \citep{zhang2025partialautomata}.  This paper converts that perspective into a kernel-compatible recurrent attention layer: cyclic and dihedral state tracking become exact phase-memory constructions, while general finite-state one-hot expressivity is stated separately for the generalized tied-write affine template.

\paragraph{Retrieval, copying, and hybrid architectures.}
A central limitation of pure fixed-state sequence models is exact retrieval and copying over long contexts \citep{jelassi2024repeat,wen2024rnn}.  Empirical studies of efficient language models emphasize the recall--throughput tradeoff \citep{arora2023zoology,arora2024simple}.  Hybrid models combine efficient recurrent or SSM layers with occasional full-attention layers to recover global selection capability \citep{lieber2024jamba,blakeman2025nemotron,kimi2025linear}.  SFDA's intended empirical target is therefore not merely a toy counter improvement, but a better quality--efficiency frontier for such hybrids.

\section{Limitations and Open Problems}
\label{sec:limitations}

This preprint deliberately avoids several overclaims.

\paragraph{No fixed-rank full-sequence theorem.}
Exact composition increases WY rank additively. The theorem is chunk-bounded, not globally fixed-rank. Any global fixed-rank scan would require additional structure or approximate recompression.

\paragraph{Finite-state tied-write expressivity is solved; compact SFDA expressivity is characterized for useful families, not universally.}
The appendix proves that every deterministic finite automaton, including every finite semidirect automaton, is exactly realized by a generalized tied-write affine template after one-hot lifting. This closes the finite-state expressivity objection. \Cref{sec:compact-expressivity} further proves compact kernel-compatible realizations for product cyclic counters, dihedral phase-orientation memory, register/reset memory, and bounded stacks when the structured control family includes the corresponding rotations, reflections, or partial permutations. What remains open is a complete low-dimensional classification for arbitrary finite automata or arbitrary semigroup transitions under the strict phase-plus-rank-one family $A_t=\Lambda_t-u_tw_t^*$.

\paragraph{Orbit-restricted tied-write routing is characterized, not solved in full generality.}
Appendix~\ref{app:tied_write_orbit} proves that tied writes can route decoupled write directions lying in the orbit of the tied key under the structured control family. The exact full-state theorem requires a scratch/protected channel or an orthogonality condition ensuring that the erase projection does not damage carried memory. Fixed block rotations are kernel-compatible but not transitive on the full sphere, so arbitrary one-step write-direction routing is not claimed. Richer routing, such as Givens products or full orthogonal control, increases expressivity but may weaken the simple chunk-WY kernel.

\paragraph{Constructive recursion is not yet a fused kernel.}
The constructive chunk-WY recursion proves that the factors can be built exactly and locally. It does not by itself guarantee that a Triton/CUDA implementation matches KDA speed. The systems contribution still requires implementation, profiling, memory-traffic analysis, and comparison to KDA kernels.

\paragraph{No empirical superiority yet.}
The package states the required experiments and target claims. It should not say SFDA beats Kimi Linear until the hybrid-ratio and throughput experiments have been run under fair compute.

\bibliographystyle{plainnat}
\bibliography{references}

\appendix
\section{Additional Proofs}
\label{app:proofs}

\subsection{Real block rotations and complex diagonal phases}

\begin{lemma}[Real-to-complex phase equivalence]
\label{lem:real-complex}
Let $z=x+iy$ identify $(x,y)\in\R^2$ with $z\in\C$.  The real block $\alpha R(\theta)$ acts as multiplication by $\alpha e^{i\theta}$:
\begin{equation}
\alpha R(\theta)
\begin{bmatrix}x\\y\end{bmatrix}
\quad\leftrightarrow\quad
\alpha e^{i\theta}z.
\end{equation}
Therefore, if $U_t=\oplus_jR(\theta_{t,j})$ and $D_t=\oplus_j\alpha_{t,j}I_2$, then $U_tD_t$ is represented in complex coordinates by $\Lambda_t=\diag(\alpha_t\od e^{i\theta_t})$.
\end{lemma}

\begin{proof}
By direct multiplication,
\begin{align}
\alpha R(\theta)\begin{bmatrix}x\\y\end{bmatrix}
&=\alpha\begin{bmatrix}x\cos\theta-y\sin\theta\\x\sin\theta+y\cos\theta\end{bmatrix},
\end{align}
which is the real and imaginary part of $\alpha(\cos\theta+i\sin\theta)(x+iy)$.
\end{proof}

\subsection{Block-WY identity with dimensions}

The closure theorem in \cref{thm:wy-closure-main} is valid over $\R$ with transpose or over $\C$ with conjugate transpose.  For clarity, suppose
\begin{align}
Y_1,W_1&\in\F^{d\times r_1}, & M_1&\in\F^{r_1\times r_1},\\
Y_2,W_2&\in\F^{d\times r_2}, & M_2&\in\F^{r_2\times r_2}.
\end{align}
Then $Y_{21},W_{21}\in\F^{d\times(r_1+r_2)}$ and $M_{21}\in\F^{(r_1+r_2)\times(r_1+r_2)}$.  The lower-left block $-M_2W_2^*Y_1M_1$ has shape $r_2\times r_1$, so the block matrix is well-typed.

\subsection{Affine composition and chunk scans}

\begin{lemma}[Affine maps form a semidirect product]
\label{lem:affine-semidir}
Let $V=\F^{d\times d_v}$ be an additive vector space and let $\mathcal{G}\subseteq\F^{d\times d}$ be a multiplicative semigroup acting on $V$ by left multiplication.  The set $V\rtimes\mathcal{G}$ with multiplication
\begin{equation}
(B_2,A_2)(B_1,A_1)=(A_2B_1+B_2,\ A_2A_1)
\end{equation}
is a semigroup.  It is a monoid if $\mathcal{G}$ contains $I$.
\end{lemma}

\begin{proof}
This is the usual semidirect product law.  Associativity follows by expanding both $((B_3,A_3)(B_2,A_2))(B_1,A_1)$ and $(B_3,A_3)((B_2,A_2)(B_1,A_1))$:
\begin{equation}
(A_3A_2B_1+A_3B_2+B_3,\ A_3A_2A_1).
\end{equation}
The identity is $(0,I)$ when $I\in\mathcal{G}$.
\end{proof}

\begin{corollary}[Chunk parenthesization]
For any partition of $\{1,\ldots,T\}$ into contiguous chunks, composing affine maps within chunks and then composing the chunk summaries gives exactly the same affine map as token-by-token recurrence.
\end{corollary}

\begin{proof}
It is just associativity of \cref{lem:affine-semidir}.
\end{proof}

\subsection{Full proof of the constructive WY recursion}
\label{app:constructive-proof}

We prove \cref{thm:constructive-wy} in detail. The base case $t=1$ gives
\begin{equation}
\Gamma_1=\Lambda_1,
\qquad
Y_1=u_1,
\qquad
W_1=r_1,
\qquad
M_1=[1],
\end{equation}
hence $P_1=\Lambda_1-u_1r_1^*=\Gamma_1-Y_1M_1W_1^*$.

Assume $P_{t-1}=\Gamma_{t-1}-Y_{t-1}M_{t-1}W_{t-1}^*$. Then
\begin{align}
P_t
&=(\Lambda_t-u_tr_t^*)(\Gamma_{t-1}-Y_{t-1}M_{t-1}W_{t-1}^*)\\
&=\Lambda_t\Gamma_{t-1}-\Lambda_tY_{t-1}M_{t-1}W_{t-1}^*\nonumber\\
&\quad-u_tr_t^*\Gamma_{t-1}+u_tr_t^*Y_{t-1}M_{t-1}W_{t-1}^*.
\label{eq:appendix-constructive-expand}
\end{align}
Using the proposed updates,
\begin{align}
Y_tM_tW_t^*
&=\begin{bmatrix}\Lambda_tY_{t-1}&u_t\end{bmatrix}
\begin{bmatrix}
M_{t-1}&0\\
-r_t^*Y_{t-1}M_{t-1}&1
\end{bmatrix}
\begin{bmatrix}W_{t-1}^*\\ r_t^*\Gamma_{t-1}\end{bmatrix}\\
&=\Lambda_tY_{t-1}M_{t-1}W_{t-1}^*
-u_tr_t^*Y_{t-1}M_{t-1}W_{t-1}^*\nonumber\\
&\quad+u_tr_t^*\Gamma_{t-1}.
\label{eq:appendix-constructive-lowrank}
\end{align}
Substituting \cref{eq:appendix-constructive-lowrank} into $\Gamma_t-Y_tM_tW_t^*$ with $\Gamma_t=\Lambda_t\Gamma_{t-1}$ gives exactly \cref{eq:appendix-constructive-expand}. Thus $P_t=\Gamma_t-Y_tM_tW_t^*$ by induction.

\section{Constructive Algorithm Details}
\label{app:algorithms}

This appendix expands the chunk-WY recursion in code-like form and records the implementation conventions needed for an arXiv preprint. Throughout, $r_t$ denotes the right WY vector in $A_t=\Lambda_t-u_tr_t^*$; it is not the attention value vector.

\subsection{Left-to-right product convention}

We use the left-to-right chunk product
\begin{equation}
P_t=A_tA_{t-1}\cdots A_1.
\end{equation}
The convention matters because the factors in \cref{thm:constructive-wy} are asymmetric. Reversing the product order changes the placement of $\Gamma_{t-1}^*r_t$ and the lower-triangular structure of $M_t$.

\subsection{Base case and dimensions}

At $t=1$,
\begin{equation}
\Gamma_1=\Lambda_1,
\qquad
Y_1=u_1,
\qquad
W_1=r_1,
\qquad
M_1=[1],
\end{equation}
so
\begin{equation}
P_1=\Lambda_1-u_1r_1^*=\Gamma_1-Y_1M_1W_1^*.
\end{equation}
If $Y_{t-1},W_{t-1}\in\F^{d\times(t-1)}$ and $M_{t-1}\in\F^{(t-1)\times(t-1)}$, then
\begin{equation}
r_t^*Y_{t-1}M_{t-1}\in\F^{1\times(t-1)}.
\end{equation}
Thus the update
\begin{equation}
M_t=
\begin{bmatrix}
M_{t-1}&0\\
-r_t^*Y_{t-1}M_{t-1}&1
\end{bmatrix}
\end{equation}
is well typed. The negative sign is essential: it cancels the cross term generated by multiplying the new rank-one correction into the previous low-rank correction.

\subsection{Reference pseudocode}

The following pseudocode is intended as a reference implementation, not as the final optimized Triton kernel.

\begin{verbatim}
Gamma = identity_phase_or_diag()
Y = empty_matrix(d, 0)
W = empty_matrix(d, 0)
M = empty_matrix(0, 0)

for t in range(C):
    Lambda_t = phase_decay[t]          # diagonal or 2x2 block phase
    u_t = beta[t] * k[t]
    r_t = Lambda_t.conj().T @ k[t]     # elementwise for diagonal phase

    Gamma_old = Gamma
    Y_old, W_old, M_old = Y, W, M

    Gamma = Lambda_t @ Gamma_old
    Y = concat_cols(Lambda_t @ Y_old, u_t)
    W = concat_cols(W_old, Gamma_old.conj().T @ r_t)

    row = - r_t.conj().T @ Y_old @ M_old
    M = block([[M_old, zeros],
               [row,   ones]])
\end{verbatim}

For diagonal complex $\Lambda_t$, the multiplications $\Lambda_tY$ and $\Gamma^*r_t$ are elementwise over channels. For real block-rotation implementation, they are independent $2\times2$ block multiplications.

\subsection{Affine write term}

The recurrent update is
\begin{equation}
S_t=A_tS_{t-1}+B_t,
\qquad
B_t=\beta_tk_ta_t^*,
\end{equation}
where $a_t$ is the attention value vector. Inside a chunk,
\begin{equation}
S_C=A_{C:1}S_0+\sum_{i=1}^C A_{C:i+1}B_i.
\end{equation}
The additive summary can be computed without forming any dense $A_{C:i}$ by running the same recurrence with $\bar S_0=0$:
\begin{equation}
\bar S_t=A_t\bar S_{t-1}+B_t,
\qquad
B_C^{\rm chunk}=\bar S_C.
\end{equation}

\subsection{Why rank remains bounded}

For a chunk of size $C$, the recursion stops at $t=C$, so $Y_C,W_C\in\F^{d\times C}$ and $M_C\in\F^{C\times C}$. Across chunks, the recommended implementation does not compose $(Y,W,M)$ factors into a larger global factor. Instead it applies each chunk transfer to the boundary state:
\begin{equation}
S_{j+1}=\Gamma_jS_j-Y_jM_j(W_j^*S_j)+B_j.
\end{equation}
This is exact and avoids full-sequence rank growth.

\subsection{Optimization target}

The reference construction materializes $M\in\F^{C\times C}$ and costs $\bigO(dC^2+C^3)$ per chunk. For $C\le64$, this is a reasonable first implementation. The optimized kernel should avoid slow Python-level loops and instead:
\begin{enumerate}[leftmargin=*]
    \item compute phase prefix products by elementwise complex multiplication;
    \item compute the within-chunk Gram-like matrix $r_i^*\Gamma_{i-1:j-1}u_j$ using batched matrix multiplication;
    \item construct or solve the triangular $M$ system by forward substitution;
    \item use Tensor-Core-friendly matmuls for $W^*S$, $M(\cdot)$, and $Y(\cdot)$.
\end{enumerate}
This is the code-level systems gap remaining after the theorem.

\section{Kernel Details and Rank-Growth Discipline}
\label{app:kernel-details}

\subsection{What rank growth means}

If a chunk summary has rank $C$, composing two exact chunk summaries produces rank at most $2C$ by \cref{thm:wy-closure-main}.  Repeating this across $T/C$ chunks gives rank $T$ in the worst case.  This is not useful as a full-sequence kernel state.

The correct discipline is:
\begin{enumerate}[leftmargin=*]
    \item Build exact low-rank factors only inside chunks.
    \item Store chunk summaries $\transfer_j=(\Gamma_j,Y_j,W_j,M_j,B_j)$ with rank at most $C$.
    \item Apply each chunk summary to the boundary state; do not materialize a global $YMW^*$.
\end{enumerate}

\subsection{Boundary-state scan}

For chunk $j$,
\begin{equation}
S_{j+1}=\Gamma_jS_j-Y_jM_j(W_j^*S_j)+B_j.
\end{equation}
The operation uses matrix products with shapes
\begin{align}
W_j^*S_j&\in\F^{C\times d_v},\\
M_j(W_j^*S_j)&\in\F^{C\times d_v},\\
Y_jM_j(W_j^*S_j)&\in\F^{d\times d_v}.
\end{align}
Thus the dominant transfer cost is $\bigO(dCd_v+C^2d_v)$ plus diagonal/block-diagonal multiplication by $\Gamma_j$.

\subsection{Potential optimized recurrence for factors}

The closure theorem gives a correct but naive way to build $Y,W,M$.  An optimized kernel should avoid explicitly applying the generic closure formula at every step.  The expected path is analogous to KDA:
\begin{enumerate}[leftmargin=*]
    \item compute phase prefix products $\Gamma_{i\to j}=\prod_{t=i}^{j}\Lambda_t$ by elementwise complex multiplication;
    \item compute pairwise inner products $w_i^*\Gamma_{i+1\to j-1}u_j$ inside the chunk;
    \item form a lower-triangular system for $M$ or its inverse;
    \item use batched matmuls for $W^*S$, $M(\cdot)$, and $Y(\cdot)$.
\end{enumerate}
This is the systems work that turns the theorem into a KDA-speed operator.

\subsection{Numerical considerations}

Phase factors have unit modulus, while decays satisfy $0\le \alpha_{t,j}\le1$.  Products can underflow for long chunks if decays are small.  As in gated linear attention and KDA-style kernels, one may store log-decays and apply local renormalization inside chunks.  The phase part can be stored either as $(\cos\theta,\sin\theta)$ or as complex half/bfloat16 pairs depending on hardware support.

\subsection{Claim language for the paper}

Safe claim:
\begin{quote}
The phase-augmented delta transition admits an exact block-WY representation within chunks.  Since chunk size is fixed, rank growth is bounded inside each chunk.  Chunk summaries can be applied to boundary states without dense transition matrices, preserving recurrent inference and enabling chunk-parallel training.
\end{quote}

Unsafe claim:
\begin{quote}
The full sequence admits an exact fixed-rank low-rank correction independent of $T$.
\end{quote}

\subsection{Derivation plan for the phase-augmented kernel}
\label{app:kernel-stage-plan}

A useful way to read the kernel theorem is as a staged restriction of a dense affine scan.

\paragraph{Stage 1: phase decay only.}
If $A_t=\Lambda_t$ with $\Lambda_t=\diag(\alpha_t\od e^{i\theta_t})$, then
\begin{equation}
A_tA_{t-1}\cdots A_1=\prod_{i=1}^{t}\Lambda_i,
\end{equation}
which is represented by elementwise complex multiplication.  This is the reason for using aligned $2\times2$ rotations rather than arbitrary dense orthogonal matrices.

\paragraph{Stage 2: one delta correction.}
For
\begin{equation}
A_t=\Lambda_t-u_tr_t^*,
\end{equation}
products remain diagonal/block-phase plus low-rank corrections.  The $M$ factor is needed because cross terms such as $u_2r_2^*u_1r_1^*$ must be represented without expanding into dense matrices.  This is exactly what the block-WY invariant stores.

\paragraph{Stage 3: chunk-bounded exactness.}
The exact chunk theorem is
\begin{equation}
A_CA_{C-1}\cdots A_1=\Gamma_C-Y_CM_CW_C^*.
\end{equation}
Since $C$ is fixed, the rank is bounded by $C$.  The theorem should never be stated with $C=T$ unless the intended cost is full-sequence rank growth.

\paragraph{Stage 4: boundary-state scan.}
For the affine recurrence $S_t=A_tS_{t-1}+B_t$, the chunk summary is
\begin{equation}
\mathcal T_j=(\Gamma_j,Y_j,W_j,M_j,B_j),
\end{equation}
where
\begin{equation}
S_{j+1}=\Gamma_jS_j-Y_jM_j(W_j^*S_j)+B_j.
\end{equation}
The recommended exact implementation does not compose low-rank factors across chunks.  It scans boundary states.  If full inter-chunk parallelism is needed, one may use superchunks or approximate recompression, but those are separate implementation choices.

\subsection{Claim language for this preprint}
\label{app:preprint-claim-language}

The defensible kernel claim is:
\begin{quote}
The phase-augmented delta transition admits an exact block-WY representation within fixed-size chunks.  Because the control is diagonal in complex coordinates, the chunk product can be represented by phase-prefix products plus rank-$C$ correction factors.  Chunk summaries can be applied to boundary states without materializing dense $d\times d$ transition matrices.
\end{quote}

The claim that still requires implementation evidence is:
\begin{quote}
A fused Triton/CUDA implementation reaches KDA-like throughput and memory traffic.
\end{quote}

The claim that should not be made is:
\begin{quote}
The full sequence has an exact fixed-rank representation independent of $T$.
\end{quote}

\subsection{Chunk-kernel implementation checklist}
\label{app:chunk-kernel-checklist}

A reference implementation should follow the same rank discipline used by KDA-like chunk kernels: never set the chunk size equal to the full sequence length.  Fix a small chunk length
\begin{equation}
C\in\{64,128,256\}.
\end{equation}
Inside one chunk, construct
\begin{equation}
A_{C:1}=A_CA_{C-1}\cdots A_1=\Gamma_C-Y_CM_CW_C^*,
\end{equation}
with $Y_C,W_C\in\F^{d\times C}$ and $M_C\in\F^{C\times C}$.  The rank is therefore bounded by $C$ inside the chunk and is never allowed to grow to $T$.

For each chunk $j$, store the transfer object
\begin{equation}
\mathcal T_j=(\Gamma_j,Y_j,W_j,M_j,B_j),
\end{equation}
which acts on a boundary state by
\begin{equation}
S_{j+1}=\Gamma_jS_j-Y_jM_j(W_j^*S_j)+B_j.
\end{equation}
This is the exact boundary-state scan.  It avoids global low-rank composition and therefore avoids the rank-doubling problem that would arise from composing two rank-$C$ summaries into rank $2C$ summaries.

\paragraph{Reference per-chunk steps.}
A first PyTorch/Triton prototype should implement the following operations:
\begin{enumerate}[leftmargin=*]
    \item form phase-decay blocks $\Lambda_t=\diag(\alpha_t\od e^{i\theta_t})$ or the equivalent real $2\times2$ blocks;
    \item build $u_t=\beta_t k_t$ and $r_t=\Lambda_t^*k_t$;
    \item update $(\Gamma,Y,W,M)$ using \cref{thm:constructive-wy};
    \item compute the affine write summary by the zero-state recurrence $\bar S_0=0$, $\bar S_t=A_t\bar S_{t-1}+B_t$;
    \item apply the transfer by $W^*S$, then $M(\cdot)$, then $Y(\cdot)$, plus the cheap phase multiplication $\Gamma S$.
\end{enumerate}
The naive factor build is $\bigO(dC^2+C^3)$ per chunk, while transfer application is $\bigO(dCd_v+C^2d_v+dd_v)$.  For small $C$, this is a usable correctness kernel.  A fused kernel should reduce memory traffic and approach KDA-style $\bigO(TdC)$ behavior in practice, but that is an implementation benchmark rather than a theorem.

\paragraph{Exactness versus parallelism.}
There are three ways to process multiple chunks:
\begin{description}[leftmargin=*]
    \item[Boundary-state scan.] Exact and simplest.  Do not compose low-rank summaries; apply each transfer to the incoming state.
    \item[Superchunks.] Exact inside a larger window $C_s=mC$, with rank bounded by $C_s$.
    \item[Recompression.] Compose summaries and compress back to rank $C$ by QR/SVD/randomized low-rank approximation.  This can increase inter-chunk parallelism, but it is approximate and should be reported separately.
\end{description}

\section{Finite-State Expressivity and Automata Interpretation}
\label{app:automata}

This appendix gives the finite-state expressivity result in a way that does not assume automata-theory background. The main point is simple: a finite-state machine can always be represented by one-hot vectors, and its transition rule can always be represented by a deterministic matrix. Therefore, in the generalized tied-write affine template, the erase-write coupling is not an expressivity obstruction for finite automata: we can set the write coefficient to zero and use only the linear transition.

\subsection{The bounded-overhead question}
\label{app:problem-b}

A central question raised by KDA and SFDA is whether tying the erase and write directions reduces expressive power. A write-decoupled affine memory system has transitions
\begin{equation}
x_t=A_{\sigma_t}x_{t-1}+b_{\sigma_t},
\qquad x_t\in\R^d,
\label{eq:app-decoupled-affine}
\end{equation}
where the additive write $b_\sigma$ is independent of the linear transition $A_\sigma$. A generalized tied-write system has transitions
\begin{equation}
z_t=(I-\beta_{\sigma_t}k_{\sigma_t}k_{\sigma_t}^*)
\widetilde A_{\sigma_t}z_{t-1}
+
\beta_{\sigma_t}k_{\sigma_t}r_{\sigma_t},
\qquad z_t\in\R^D.
\label{eq:app-general-tied}
\end{equation}
Here the same vector $k_\sigma$ controls both the erase direction and the write direction.

\begin{problem}[Bounded-overhead tied-write realization]
Given a write-decoupled affine memory system in dimension $d$, does there exist a dimension $D$, an injective encoding $\phi:\R^d\to\R^D$, and tied-write transitions $\widetilde T_\sigma$ of the form \cref{eq:app-general-tied} such that
\begin{equation}
\phi(A_\sigma x+b_\sigma)=\widetilde T_\sigma(\phi(x))
\end{equation}
for every symbol $\sigma$ and every reachable state $x$? Can $D$ be bounded by $\operatorname{poly}(d)$?
\end{problem}

The theorem below answers the finite-state version of this question. It does not solve the stronger low-dimensional compression question for general continuous affine systems.

\subsection{Finite automata as one-hot linear systems}
\label{app:finite-state-proof}

A deterministic finite automaton consists of a finite set of states
\begin{equation}
Q=\{q_1,\ldots,q_N\},
\end{equation}
a finite input alphabet $\Sigma$, and for every input symbol $\sigma\in\Sigma$, a transition map
\begin{equation}
\delta_\sigma:Q\to Q.
\end{equation}
If the current state is $q$ and symbol $\sigma$ is read, the next state is $\delta_\sigma(q)$.

\paragraph{One-hot encoding.}
Encode state $q_i$ by the standard basis vector
\begin{equation}
\phi(q_i)=e_i\in\R^N.
\end{equation}
For each input symbol $\sigma$, define the matrix $P_\sigma\in\R^{N\times N}$ by
\begin{equation}
P_\sigma e_i=e_j
\quad\text{whenever}\quad
\delta_\sigma(q_i)=q_j.
\label{eq:app-Psigma}
\end{equation}
Thus $P_\sigma$ is exactly the linear operator that implements the automaton transition in one-hot coordinates.

\begin{remark}[Not necessarily a permutation]
Each column of $P_\sigma$ contains exactly one $1$, because every current state has exactly one next state. But $P_\sigma$ need not be a permutation matrix: two different states may transition to the same next state. Reset maps and absorbing states are typical examples.
\end{remark}

\begin{lemma}[One-hot linear realization]
\label{lem:onehot-linear}
For every input sequence $\sigma_1,\ldots,\sigma_T$, if
\begin{equation}
q_t=\delta_{\sigma_t}(q_{t-1}),
\end{equation}
and
\begin{equation}
s_0=\phi(q_0),
\qquad
s_t=P_{\sigma_t}s_{t-1},
\end{equation}
then $s_t=\phi(q_t)$ for all $t=0,\ldots,T$.
\end{lemma}

\begin{proof}
The claim holds for $t=0$ by definition. Suppose $s_{t-1}=\phi(q_{t-1})$. If $q_{t-1}=q_i$, then $s_{t-1}=e_i$. Let $q_t=\delta_{\sigma_t}(q_i)=q_j$. By construction of $P_{\sigma_t}$,
\begin{equation}
s_t=P_{\sigma_t}e_i=e_j=\phi(q_j)=\phi(q_t).
\end{equation}
The result follows by induction.
\end{proof}

\subsection{Exact generalized tied-write realization of finite automata}

We now insert the one-hot linear recurrence into the generalized tied-write template \cref{eq:app-general-tied}. For finite-state simulation, the tied erase-write constraint causes no obstruction, because we can disable the write channel.

\begin{theorem}[Exact finite-state realization by generalized tied-write updates]
\label{thm:app-exact-tied-finite}
Every deterministic finite automaton with $N$ states admits an exact realization by the generalized tied-write update \cref{eq:app-general-tied} in dimension $N$.
\end{theorem}

\begin{proof}
Use the one-hot encoding $\phi(q_i)=e_i\in\R^N$. For each symbol $\sigma$, define $P_\sigma$ as in \cref{eq:app-Psigma}. Given an input sequence $\sigma_1,\ldots,\sigma_T$, choose
\begin{equation}
\widetilde A_{\sigma_t}=P_{\sigma_t},
\qquad
\beta_{\sigma_t}=0.
\end{equation}
Then \cref{eq:app-general-tied} reduces to
\begin{equation}
z_t=P_{\sigma_t}z_{t-1}.
\end{equation}
By \cref{lem:onehot-linear}, $z_t=\phi(q_t)$ for all $t$. Thus the generalized tied-write update exactly realizes the automaton.
\end{proof}

\paragraph{Why the proof is so short.}
The tied-write mechanism only matters when $\beta_\sigma\neq0$. By setting $\beta_\sigma=0$, both the erase term and the write term vanish, and the system becomes an ordinary linear recurrence. The finite automaton is already an ordinary linear recurrence after one-hot encoding. Hence the tied-write coupling is irrelevant for exact finite-state realization.

\subsection{Finite semidirect automata}

A finite semidirect automaton has state space
\begin{equation}
Q=N\rtimes H,
\end{equation}
where $N$ is a finite memory set and $H$ is a finite control set. A state is a pair $q=(n,h)$, where $n\in N$ is the memory component and $h\in H$ is the control component. A typical semidirect update has the form
\begin{equation}
(n,h)\mapsto (n+\rho(h)a_\sigma,\;hh_\sigma),
\label{eq:app-semidir-update}
\end{equation}
where $a_\sigma\in N$, $h_\sigma\in H$, and $\rho(h)$ denotes the action of $h$ on $N$. The precise algebraic form is not needed for the theorem below. The only required fact is that $N\rtimes H$ is finite.

\begin{corollary}[Exact finite semidirect realization]
\label{cor:app-semidir-finite}
Every finite semidirect automaton with state space $N\rtimes H$ is exactly realizable by the generalized tied-write update in dimension
\begin{equation}
D=|N||H|.
\end{equation}
\end{corollary}

\begin{proof}
Since $N$ and $H$ are finite, $Q=N\rtimes H$ is a finite state set with $|Q|=|N||H|$. Every input symbol induces a deterministic transition map $\delta_\sigma:Q\to Q$. Apply \cref{thm:app-exact-tied-finite}.
\end{proof}

\subsection{What remains open after the solution}
\label{app:what-remains-open}

\Cref{thm:app-exact-tied-finite,cor:app-semidir-finite} solve the finite-state expressivity objection for the generalized tied-write template. The solution is exact and elementary. However, it should not be confused with two stronger claims.

\paragraph{Not a low-dimensional compression theorem.}
The one-hot construction uses dimension equal to the number of reachable states:
\begin{equation}
D=|Q|.
\end{equation}
If a memory system has exponentially many reachable states relative to a smaller continuous representation, one-hot lifting may be exponentially large. Therefore the theorem proves exact finite-state realizability, while the stronger question of efficient low-dimensional compilation remains separate.

\paragraph{Not a kernel-compatible SFDA universality theorem.}
The proof uses unrestricted transition matrices $\widetilde A_\sigma=P_\sigma$. The kernel-compatible tied SFDA transition has the much more structured form
\begin{equation}
A_t=(I-\beta_tk_tk_t^*)\Lambda_t=\Lambda_t-u_tw_t^*,
\end{equation}
where $\Lambda_t$ is diagonal or block-phase. This restriction is what makes the constructive chunk-WY kernel possible. The exact one-hot theorem therefore removes finite automata as an expressivity obstruction, but it does not say that every finite automaton admits a small phase-plus-rank-one SFDA realization.

\paragraph{Clean final interpretation.}
The finite-state question is solved: tied-write coupling does not prevent exact one-hot realization. The remaining research problem is sharper and more useful for systems evaluation: can structured SFDA transitions give low-dimensional, hardware-efficient realizations that improve the KDA quality--efficiency frontier?

\section{Experiment Protocol}
\label{app:experiment-protocol}

\subsection{Synthetic automata}

\begin{table}[h]
\centering
\caption{Automata suite.  Train lengths should be much shorter than test lengths.}
\begin{tabular}{llll}
\toprule
Task & Train length & Test length & Main metric \\
\midrule
$C_m$ counter & 256--1024 & 2048--16384 & final/prefix accuracy \\
$D_{2m}$ dihedral & 256--1024 & 2048--16384 & state accuracy \\
Reset counter & 256--1024 & 2048--16384 & post-reset accuracy \\
Bounded Dyck & 256--1024 & 2048--8192 & stack top accuracy \\
Kimi Stack & 256--2048 & 4096--8192 & pop accuracy \\
MQAR & 1024--4096 & 8192--32768 & recall accuracy \\
Palindrome & 256--2048 & 4096--8192 & reverse-copy accuracy \\
\bottomrule
\end{tabular}
\end{table}

\subsection{Language modeling}

Use at least three model scales if compute permits: 100M, 300M, and 1B parameters.  Compare:
\begin{itemize}[leftmargin=*]
    \item full attention / MLA-style baseline;
    \item KDA hybrid at $3{:}1$;
    \item SFDA hybrid at $3{:}1$, $7{:}1$, and $15{:}1$;
    \item ablations: no phase, fixed phase, learned phase, random phase, phase without delta.
\end{itemize}

\subsection{Kernel benchmarks}

Report forward and backward time for $T\in\{4k,16k,64k,128k\}$, chunk $C\in\{64,128\}$, and $d_k=d_v=128$.  Include memory footprint and achieved TFLOP/s.  The relevant comparison is not against dense DPLR only, but against the best available KDA kernel.

\subsection{Long-context benchmarks}

The final paper should include at least RULER, MRCR, RepoQA, LongBench v2, and a long-code benchmark.  The hybrid-ratio result is the key table:
\begin{center}
\begin{tabular}{lcccc}
\toprule
Model & Ratio & RULER & RepoQA & Avg. \\
\midrule
KDA & 3:1 & -- & -- & -- \\
SFDA & 3:1 & -- & -- & -- \\
SFDA & 7:1 & -- & -- & -- \\
SFDA & 15:1 & -- & -- & -- \\
\bottomrule
\end{tabular}
\end{center}

\section{Tied-Write Compilation over Structured Control Orbits}
\label{app:tied_write_orbit}

This appendix gives a precise restricted statement of what tied-write SFDA can realize without appealing to arbitrary dense transition matrices.  The key idea is an \emph{orbit condition}: a tied write into a key direction $k$ can be routed into another direction $w$ when $w$ is reachable from $k$ by the structured control operators available to the architecture.

There is one important technical caveat.  In the KDA/SFDA delta rule, a tied write is not a pure additive write.  It also contains an erase projection.  Therefore the full affine compilation theorem needs a protected-subspace assumption ensuring that the erase projection does not modify the carried memory.  We state both the contribution-level routing lemma and the full affine theorem below.

\subsection{Model classes}

Let $x\in\R^d$ denote a memory state.  A decoupled structured affine update has the form
\begin{equation}
    x^+=G_\sigma x+w_\sigma r_\sigma,
    \label{eq:decoupled-update-orbit}
\end{equation}
where $G_\sigma\in\R^{d\times d}$ is a structured transition operator, $w_\sigma\in\R^d$ is a write direction, and $r_\sigma\in\R$ is the scalar value being written.  The write direction $w_\sigma$ is independent of the direction used by the erase/read mechanism.

A tied-write delta step has the form
\begin{equation}
    z^+=(I-\beta_\sigma kk^\top)\widetilde G_\sigma z+\beta_\sigma k r_\sigma,
    \label{eq:tied-update-orbit}
\end{equation}
where the same vector $k$ controls the erase direction and the write direction.  This is the structural coupling present in KDA-style delta memory.

\subsection{Structured control family and orbits}

Let $\mathcal G\subseteq\R^{d\times d}$ be the family of structured control operators allowed by the architecture.  Examples include diagonal phase-decay matrices, block rotations, signed permutations, Givens rotations, or products of such operators.

For a fixed nonzero key $k\in\R^d$, define its orbit under $\mathcal G$ as
\begin{equation}
    \operatorname{Orb}_{\mathcal G}(k)=\{Rk:R\in\mathcal G\}.
    \label{eq:orbit}
\end{equation}
The orbit is the set of directions reachable from the tied key using only the structured control operators available to the model.

\subsection{Contribution-level routing}

\begin{lemma}[Orbit routing of a write contribution]
\label{lem:contribution-routing}
Let $k\in\R^d$ be a tied-write key and let $\mathcal G$ be the structured control family.  Suppose a desired write direction $w_\sigma$ satisfies
\begin{equation}
    w_\sigma\in\operatorname{Orb}_{\mathcal G}(k),
\end{equation}
or equivalently, there exists $R_\sigma\in\mathcal G$ such that $w_\sigma=R_\sigma k$.  Then the additive contribution $w_\sigma r_\sigma$ can be obtained by writing $r_\sigma$ along $k$ and applying $R_\sigma$ to that contribution.
\end{lemma}

\begin{proof}
Since $w_\sigma=R_\sigma k$, a tied write creates the contribution $kr_\sigma$.  Applying $R_\sigma$ gives
\begin{equation}
R_\sigma(kr_\sigma)=(R_\sigma k)r_\sigma=w_\sigma r_\sigma.
\end{equation}
\end{proof}

\paragraph{Why this lemma is not yet a full update theorem.}
A model applies $R_\sigma$ to the state, not only to the newly written contribution.  Moreover, the tied delta step also contains the erase term $(I-\beta kk^\top)$.  The full theorem must account for both effects.

\subsection{Restricted tied-write compilation theorem}

\begin{theorem}[Orbit-restricted tied-write affine compilation]
\label{thm:orbit-restricted-compilation}
Let $\mathcal X\subseteq\R^d$ be the set of reachable memory states.  Consider the desired decoupled affine update
\begin{equation}
T_\sigma(x)=G_\sigma x+w_\sigma r_\sigma.
\end{equation}
Assume there exists a structured routing operator $R_\sigma\in\mathcal G$ and a structured pre-routing operator $\widetilde G_\sigma$ such that
\begin{align}
    G_\sigma&=R_\sigma\widetilde G_\sigma,
    \label{eq:orbit-linear-part}\\
    w_\sigma&=R_\sigma k,
    \label{eq:orbit-write-part}
\end{align}
and assume the protected-subspace condition
\begin{equation}
    k^\top \widetilde G_\sigma x=0
    \qquad\text{for every }x\in\mathcal X.
    \label{eq:protected-subspace}
\end{equation}
Then, with $\beta_\sigma=1$, the composition of a tied-write step with $\widetilde G_\sigma$ followed by the structured routing operator $R_\sigma$ implements $T_\sigma$ exactly on $\mathcal X$:
\begin{equation}
R_\sigma\bigl((I-kk^\top)\widetilde G_\sigma x+k r_\sigma\bigr)
=G_\sigma x+w_\sigma r_\sigma.
\end{equation}
\end{theorem}

\begin{proof}
By the protected-subspace condition, $kk^\top\widetilde G_\sigma x=0$ for every reachable $x$.  Hence
\begin{equation}
(I-kk^\top)\widetilde G_\sigma x=\widetilde G_\sigma x.
\end{equation}
Applying $R_\sigma$ to the tied-write step gives
\begin{align}
R_\sigma\bigl((I-kk^\top)\widetilde G_\sigma x+k r_\sigma\bigr)
&=R_\sigma\widetilde G_\sigma x+R_\sigma k r_\sigma\\
&=G_\sigma x+w_\sigma r_\sigma,
\end{align}
using \cref{eq:orbit-linear-part,eq:orbit-write-part}.
\end{proof}

\paragraph{Interpretation.}
The theorem identifies the condition under which the tied-write constraint is harmless for a decoupled write.  The desired write direction must be reachable from the tied key by the structured control family, and the carried memory must be protected from the erase projection.  Thus the expressivity question becomes a routing problem plus a subspace-allocation problem:
\begin{equation}
    \text{How large is }\operatorname{Orb}_{\mathcal G}(k),
    \qquad
    \text{and how cheaply can protected write lanes be allocated?}
\end{equation}

\subsection{Direction routing for common control families}

\begin{proposition}[Full orthogonal control]
\label{prop:orthogonal-routing}
If $\mathcal G=O(d)$, then for every pair of unit vectors $k,w\in\mathbb S^{d-1}$, there exists $R\in\mathcal G$ such that $Rk=w$.  Consequently,
\begin{equation}
    \operatorname{Orb}_{O(d)}(k)=\mathbb S^{d-1}.
\end{equation}
\end{proposition}

\begin{proof}
The orthogonal group acts transitively on the unit sphere.  Concretely, extend $k$ and $w$ to orthonormal bases $\{k,u_2,\ldots,u_d\}$ and $\{w,v_2,\ldots,v_d\}$.  The linear map sending the first basis to the second is orthogonal and maps $k$ to $w$.
\end{proof}

\begin{proposition}[Signed permutation control]
\label{prop:signed-permutation-routing}
Let $\mathcal G$ be the group of signed permutation matrices.  If $k=e_1$, then
\begin{equation}
    \operatorname{Orb}_{\mathcal G}(e_1)=\{\pm e_i:1\le i\le d\}.
\end{equation}
\end{proposition}

\begin{proof}
A signed permutation matrix maps each standard basis vector to another standard basis vector, possibly with a sign flip.  Thus every reachable direction from $e_1$ has the form $\pm e_i$.  Conversely, for every $\pm e_i$, there is a signed permutation matrix sending $e_1$ to $\pm e_i$.
\end{proof}

\begin{proposition}[Independent two-dimensional block rotations]
\label{prop:block-rotation-routing}
Assume $d$ is even and
\begin{equation}
    \mathcal G=SO(2)\oplus SO(2)\oplus\cdots\oplus SO(2),
\end{equation}
where each $SO(2)$ block acts only on one fixed pair of coordinates.  If $k=e_1$, then
\begin{equation}
    \operatorname{Orb}_{\mathcal G}(e_1)=\{(\cos\theta,\sin\theta,0,\ldots,0):\theta\in[0,2\pi)\}.
\end{equation}
In particular, this control family is not transitive on $\mathbb S^{d-1}$ for $d>2$.
\end{proposition}

\begin{proof}
The vector $e_1$ lies in the first two-dimensional coordinate block.  Since $\mathcal G$ is block diagonal, no operator in $\mathcal G$ mixes the first block with any other block.  Therefore $e_1$ can only move inside $\Span\{e_1,e_2\}$.  Within this plane, $SO(2)$ realizes all rotations.
\end{proof}

\begin{proposition}[Givens routing]
\label{prop:givens-routing}
If $\mathcal G$ contains Givens rotations between arbitrary coordinate pairs, then for every unit vector $w\in\mathbb S^{d-1}$, there exists a product of at most $d-1$ Givens rotations sending $e_1$ to $w$.
\end{proposition}

\begin{proof}
Write $w=(w_1,\ldots,w_d)^\top$ with $\norm{w}_2=1$.  Starting from $e_1$, successively rotate mass from coordinate $1$ into coordinates $2,\ldots,d$.  At step $j$, choose a Givens rotation in the $(1,j)$-plane so that the $j$-th coordinate matches $w_j$ while preserving the norm of the remaining unassigned coordinates.  After at most $d-1$ rotations, the resulting vector equals $w$.
\end{proof}

\subsection{Consequence for SFDA}

\begin{corollary}[SFDA tied-write expressivity over control orbits]
\label{cor:sfda-orbit}
Tied-write SFDA exactly realizes decoupled writes whose directions lie in the orbit of the tied key under the SFDA control family, provided the corresponding protected-subspace condition in \cref{thm:orbit-restricted-compilation} holds.  If the control family is transitive on the unit sphere, arbitrary unit write directions are routable.  If the control family is not transitive, tied-write SFDA realizes only the corresponding orbit-restricted class of decoupled writes without additional routing steps or state expansion.
\end{corollary}

\paragraph{What this proves.}
This appendix proves that tied-write coupling is not an obstruction for write directions reachable by the architecture's own structured control operators, once the erase term is neutralized on the carried memory subspace.

\paragraph{What this does not prove.}
This result does not prove that diagonal or fixed block-phase SFDA can realize all arbitrary write directions in $O(1)$ microsteps.  Fixed block rotations are generally not transitive on the full sphere.  Full arbitrary direction routing requires either richer control operators, such as general orthogonal routing, or a sequence of multiple local rotations.  Thus the remaining stronger question is whether a kernel-compatible SFDA control family can be made sufficiently expressive for practical routing while preserving the chunk-WY structure.

\section{Compact Expressivity Proof Details and Review-Driven Additions}
\label{app:compact-expressivity-proofs}

This appendix collects proof details and scope clarifications for \cref{sec:compact-expressivity}.  The goal is to state exactly what the kernel-compatible SFDA family can represent compactly, without collapsing the distinction between one-hot finite-state universality and low-dimensional structured realizations.

\subsection{The primitive generator viewpoint}

Fix a structured control family $\mathcal L$.  The linear part of one kernel-compatible SFDA microstep is
\begin{equation}
A=\Lambda-ur^*,\qquad \Lambda\in\mathcal L.
\end{equation}
Thus the compact transition class with chunk budget $C$ is the length-$C$ word ball in the semigroup generated by structured controls and rank-one corrections:
\begin{equation}
\mathsf{SFDA}_{\mathcal L}(d,C)
:=
\left\{\prod_{t=1}^{C'}(\Lambda_t-u_tr_t^*):0\le C'\le C,\ \Lambda_t\in\mathcal L\right\}.
\end{equation}
The constructive WY theorem says that every element of this set has a computable summary
\begin{equation}
P=\Gamma-YMW^*,
\qquad
\rank(P-\Gamma)\le C.
\end{equation}
The lower bound in \cref{prop:no-global-fixed-rank} shows that this rank budget cannot generally be made independent of the number of distinct correction directions.

\subsection{Readout for compact finite phase states}

The compact phase realizations store discrete states as points on a torus.  For example, $C_m$ is stored as
\begin{equation}
z_j=e^{2\pi i j/m},\qquad j\in\mathbb Z_m.
\end{equation}
This is an exact representation of the transition dynamics.  To recover a symbolic label, one may use any finite classifier on the $m$ separated points.  A small MLP or nearest-phase classifier suffices because the set of possible hidden states is finite and separated by angle $2\pi/m$.  For products $\prod_j C_{m_j}$, the readout can decode each coordinate separately or classify the product state jointly.

\subsection{Why register memories are genuinely tied-SFDA, not only generalized-template systems}

The register theorem uses the actual tied update.  For register $i$,
\begin{equation}
S^+=(I-e_ie_i^*)S+e_ia^*.
\end{equation}
This is exactly tied SFDA with $\Lambda=I$, $\beta=1$, $k=e_i$, and $v=a$.  The same key $e_i$ erases and writes.  The coupling is not harmful because a register update is precisely an erase-then-write operation.

\subsection{Partial permutations and compatibility with the WY proof}

The default implementation uses complex diagonal $\Lambda_t$.  However, the algebra of \cref{thm:constructive-wy} only requires that $\Lambda_t$ can be multiplied into $\Gamma$, applied to $Y$, and adjoint-applied to $r_t$ efficiently.  Partial permutations satisfy this requirement: multiplying or applying a sparse shift costs linear time in the number of nonzeros.  Therefore bounded stacks and shift registers are not phase-only SFDA, but they are kernel-compatible generalized SFDA under a sparse-control family.

\subsection{What would constitute a complete compact-expressivity classification?}

A full classification would characterize all finite transformation semigroups whose transition matrices admit low-dimensional embeddings in $\mathsf{SFDA}_{\mathcal L}(d,C)$ for small $d$ and $C$.  This preprint does not claim such a classification.  Instead, it proves positive results for common structured memories and a rank-budget theorem that gives the necessary computational invariant.  In symbols, the paper proves useful sufficient inclusions such as
\begin{align}
\prod_{j=1}^{\ell}C_{m_j} &\subseteq \mathsf{SFDA}_{\mathcal L_{\mathrm{phase}}}(\ell,1),\\
D_m &\subseteq \mathsf{SFDA}_{\mathcal L_{\mathrm{ortho}}}(2,1),\\
\text{register memories with }d\text{ slots} &\subseteq \mathsf{SFDA}_{\mathcal L_{\mathrm{phase}}}(d,1),\\
\text{bounded stacks of depth }h &\subseteq \mathsf{SFDA}_{\mathcal L_{\mathrm{pp}}}(h,1),
\end{align}
where the notation suppresses the value width and finite readout.  It does not prove that every finite automaton has a small representation in $\mathsf{SFDA}_{\mathcal L}(d,C)$.

\subsection{Review checklist incorporated into the main text}

The current version incorporates the following review-driven corrections and additions:
\begin{enumerate}[leftmargin=*]
    \item a model hierarchy separating KDA, tied SFDA, kernel-compatible generalized SFDA, and the generalized tied-write proof template;
    \item corrected finite-state theorem wording: one-hot finite-state realization is for the generalized tied-write affine template, not vanilla KDA;
    \item rank-budget theorem for $A_t=\Lambda_t-u_tr_t^*$ products;
    \item no-global-fixed-rank lower bound;
    \item constructive chunk-WY recursion in the main theory section;
    \item prefix-output theorem for intra-chunk token outputs;
    \item affine write-summary theorem and zero-state recurrence;
    \item realification theorem mapping complex phases to $2\times2$ real rotation-decay blocks;
    \item stability and triangular-factor conditioning lemmas;
    \item training-time parameterization constraints for $\alpha_t,\theta_t,\beta_t,k_t$;
    \item explicit relation to RoPE, SSMs, DPLR, and KDA;
    \item orbit-routing theorem with protected-subspace erase neutralization;
    \item compact cyclic, direct-product cyclic, dihedral, register/reset, and bounded-stack realization theorems;
    \item exact chunk-parallelism wording: exact intra-chunk parallelism plus recurrent or scanned boundary propagation, not fixed-rank full-sequence closure;
    \item full complexity accounting with value width $d_v$;
    \item a theorem-to-evidence map for experiments.
\end{enumerate}


\section{Details for Basic Experiments}
\label{app:toy-experiment-details}

\paragraph{Scope.}
The experiments in Section~\ref{sec:basic-experiments} are designed to test the
paper's algebraic claims and the isolated phase-memory mechanism.  They are not
large-scale language-model experiments.  In particular, they do not validate the
hybrid SFDA-to-global-attention ratio, long-context language-model quality, or
wall-clock speed of a fused Triton/CUDA kernel.

\subsection{Numerical verification of algebraic identities}
\label{app:theorem-verification}

For each theorem check, we compare the structured implementation against dense
matrix ground truth.  Table~\ref{tab:theorem-checks} reports worst-case relative
error across the tested random seeds and shapes.

\begin{table}[h]
\centering
\caption{Numerical verification of SFDA algebra.}
\label{tab:theorem-checks}
\begin{tabular}{lc}
\toprule
Check & Worst-case error \\
\midrule
Block-WY closure, Theorem~\ref{thm:wy-closure-main} & $6.7\times 10^{-16}$ \\
Constructive chunk-WY product, Theorem~\ref{thm:constructive-wy} & $1.9\times 10^{-15}$ \\
Affine chunk transfer, Theorem~\ref{thm:affine-chunk-main} & $2.4\times 10^{-16}$ \\
KDA recovered at $\theta=0$, Proposition~\ref{prop:kda-special} & $0$ \\
Cyclic phase norm drift, Proposition~\ref{prop:phase-memory}, $T=16384$ & $3.0\times 10^{-14}$ \\
Cyclic phase modular error, Proposition~\ref{prop:phase-memory} & $3.0\times 10^{-12}$ \\
Spectral stability bound, Proposition~\ref{prop:stability} & holds; max $0.9977$ \\
\bottomrule
\end{tabular}
\end{table}

\paragraph{Chunk-kernel sanity check.}
For $d_k=d_v=128$, the correction rank equals the chunk size $C$ for
$C\in\{16,32,64,128\}$ in the generic random setting.  Boundary-state transfer
matches recurrent unrolling to approximately $10^{-15}$.  The naive build cost
increases superlinearly in $C$, consistent with the reference complexity
\(O(d_kC^2+C^3)\).  This motivates the paper's chunk-bounded rank discipline:
rank is allowed to grow within a chunk, but chunk summaries are not composed into
a global fixed-rank object.

\subsection{Constructed cyclic-memory separation}
\label{app:constructed-counter}

We construct a mod-$5$ counter with state
\[
    c_t = c_{t-1}+a_t \pmod 5.
\]
SFDA uses one real two-dimensional block to represent the complex phase
\[
    z_t = \exp(2\pi i c_t/5).
\]
Equivalently, in real coordinates,
\[
    z_t = R(2\pi a_t/5) z_{t-1},
    \qquad
    R(\phi)=
    \begin{bmatrix}
        \cos\phi & -\sin\phi\\
        \sin\phi & \cos\phi
    \end{bmatrix}.
\]
The decoder maps $z_t$ to the nearest of the five phase prototypes.  Since
$R(2\pi/5)^5=I$ and rotations preserve norm, this representation is exact up to
floating-point error for arbitrary length.

The real-diagonal baseline is calibrated at the same initial length but lacks a
nontrivial period-$5$ bounded orbit.  A real scalar or diagonal decay either
contracts, expands, or flips signs along coordinates; it cannot rotate through
five distinct phase states while preserving bounded norm.  Therefore its best
calibrated classifier approaches chance level on long random sequences.

\subsection{Learned short-to-long counter}
\label{app:learned-counter}

The learned experiment trains a minimal counter head at length $48$ and evaluates
without further tuning at lengths $48,96,192,384,768$.  The SFDA model learns a
phase parameter and a decoder for the mod-$5$ state.  The KDA baseline is the
same controlled model with phase disabled, \(\theta_t\equiv0\).  This isolates
whether the phase channel, rather than additional capacity, explains the cyclic
tracking behavior.

The learned SFDA model reaches perfect train-length accuracy and remains perfect
at length $96$, then degrades gradually as small phase-estimation errors
accumulate.  This is consistent with the constructed experiment: the exact phase
parameter has no intrinsic length drift, while a learned approximate phase can
accumulate angular error over long horizons.

\subsection{Reproducibility command block}
\label{app:reproduce-basic-experiments}

The reference experiments can be reproduced with the following commands:
\begin{verbatim}
pip install numpy matplotlib torch
python3 verify_theorems.py
python3 exp_a_representational.py
python3 exp_b_learned.py
python3 plots_and_kernel.py
\end{verbatim}

\end{document}